\documentclass[preprint,12pt]{elsarticle}

\usepackage{epsfig}
\usepackage{amssymb}
\usepackage{amsthm}
\usepackage{amsmath,amsfonts,amssymb,mathrsfs}
\usepackage{algorithmic}
\usepackage{algorithm}
\usepackage{array}
\usepackage[caption=false,font=normalsize,labelfont=sf,textfont=sf]{subfig}
\usepackage{textcomp}
\usepackage{stfloats}
\usepackage{url}
\usepackage{verbatim}
\usepackage{graphicx}

\usepackage{calligra}
\usepackage{cuted}
\usepackage{blindtext}
\usepackage[colorlinks,
            linkcolor=blue,
            anchorcolor=blue,
            citecolor=blue
            ]{hyperref}

\newtheorem{theorem}{Theorem}
\newtheorem{definition}{Definition}

\journal{}

\begin{document}

\begin{frontmatter}



\title{An RRT* algorithm based on Riemannian metric model for optimal path planning}


\author[1]{Yu Zhang}
\ead{zhangyu348@hust.edu.cn}

\author[1]{Qi Zhou}
\ead{}

\author[1,2]{Xiao-Song Yang \corref{cor1}}
\ead{yangxs@hust.edu.cn}

\address[1]{School of Mathematics and Statistics, Huazhong University of Science and Technology, Wuhan 430074, China}
\address[2]{Hubei Key Laboratory of Engineering Modeling and Scientific Computing, Huazhong University of Science and Technology, Wuhan 430074, China}

\cortext[cor1]{Corresponding author}


\begin{abstract}
This paper presents a Riemannian metric-based model to solve the optimal path planning problem on two-dimensional smooth submanifolds in high-dimensional space. Our model is based on constructing a new Riemannian metric on a two-dimensional projection plane, which is induced by the high-dimensional Euclidean metric on two-dimensional smooth submanifold and reflects the environmental information of the robot. The optimal path planning problem in high-dimensional space is therefore transformed into a geometric problem on the two-dimensional plane with new Riemannian metric. Based on the new Riemannian metric, we proposed an incremental algorithm RRT*-R on the projection plane. The experimental results show that the proposed algorithm is suitable for scenarios with uneven fields in multiple dimensions. The proposed algorithm can help the robot to effectively avoid areas with drastic changes in height, ground resistance and other environmental factors. More importantly, the RRT*-R algorithm shows better smoothness and optimization properties compared with the original RRT* algorithm using Euclidean distance in high-dimensional workspace.
The length of the entire path by RRT*-R is a good approximation of the theoretical minimum geodesic distance on projection plane.

\end{abstract}




\end{frontmatter}


\section{Introduction}
\label{sec:introduction}

Motion planning has always been the core issue in robotics research\cite{siegwart2011book}\cite{siciliano2008book}\cite{motionplanningblue}\cite{MotionPlanningTRO}.
The existence problem of the path can be described by finding a route from the starting point to the desired destination without collisions to the obstacles.
Khatib proposed the Artificial Potential Fields method to address motion planning challenges\cite{khatib1986real}. The negative gradient of this field leads a way avoiding obstacles and towards the destination, but also brings the appearance of local minima which prevents the robot from safely moving to the destination. To counter this issue, Navigation Function was proposed by Koditschek and Rimon \cite{koditschek1990robot}\cite{rimon1990exact2} in the early 90’s. Subsequent research has focused on developing navigation functions for complex 3-D workspaces \cite{loizou2012complex3d}\cite{filippidis2012PartiallySufficientlyCurvedWorlds}\cite{paternainKoditschek2017navigation} and exploring variations such as those based on harmonic potentials\cite{Laplace'sEquation1990path}\cite{kim1992real}.

Among the significant amount of scientific researches on motion planning, the path optimization problem holds practical significance \cite{WildernessSearchandRescue}\cite{AMMrobotpathplanning}.
Common shortest-path algorithms, such as Dijkstra's and A*, provide solutions based on graph structures \cite{dijkstra2022note} \cite{Astar1968formal}. Recently, the Rapidly-exploring Random Tree Star (RRT*) algorithm\cite{RRT}\cite{firstRRT*2011sampling} has gained popularity in robotics and computer graphics. RRT* is a global path planning method and an enhancement of the original RRT. It explores the space by continually expanding a tree structure and optimizes the generated paths to achieve a global optimum. RRT* combines rapid exploration with path optimization, making it suitable for complex environments and high-dimensional spaces.

Path planning on a two-dimensional plane has become a well-addressed problem with the availability of various open-source libraries \cite{opensource}. Nevertheless, as long as the robot is moving on the “ground”, its motion is strictly constrained. In order to ameliorate that, Liu Ming \cite{liuming2015dingkan}defined a Riemannian metric on tensor voting vector field, which has opened a wide horizon for subsequent research \cite{ieeeaccess}\cite{lm2016deep} on optimal path planning on surfaces. Several recent researches have applied the idea of Riemannian metric in different scenarios of path planning \cite{shootingmethod}. Setting the “ground” as a Riemannian manifold allows to naturally define a local Riemannian metric that encodes the geometric information \cite{petersen2006riemannianbook}. Aziz, F. proposed a Riemannian approach for free-space extraction and path planning using color catadioptric vision. \cite{likui2020riemannian}.

However, despite the increasing deployment of surface-moving robots in various real-world outdoor applications, their performance can be significantly influenced by environmental factors, including ground resistance, temperature, wind speed and so on. Drastic changes in ground resistance can have a profound impact on the energy consumption of these robots during their movement processes \cite{GroundRobotsEnergyConsumption}.
When the friction of the surface increases, the robot must overcome greater resistance to move. This requires the motors or drive system to expend more energy to maintain the same speed. When the friction of the surface decreases, although the resistance is lower, the robot may need to frequently adjust its speed or posture to remain stable and controlled, which also consumes additional energy.
Moreover, during transitions between different ground materials, the robot might need to accelerate or decelerate to adapt to the new friction conditions. These processes consume extra energy. 

In general, the surface on which the robot moves can be regarded as a 2-D manifold denoted by height function $x_3 (x_1,x_2)$ in 3-D space. Moreover, if the environmental factors are also taken into account, the parameter space of robot motion planning will become multi-dimensional. For example, when we consider ground friction resistance, it's a function constrained to the two-dimensional surface. So we can use $x_4=x_4 (x_1,x_2)$ to represent the magnitude of the ground resistance. In general, all the other factors considered can be represented as functions $x_5, x_6, ...,x_n$ depending on $x_1,x_2$. Therefore, we can formulate the path planning problem on a two-dimensional submanifold in high-dimensional space.

In this paper, we project the two-dimensional smooth submanifold in $\mathbb{R}^n$ mentioned above onto the $(x_1,x_2)$ plane $\mathbb{R}^2$, and construct the new Riemannian metric on the two-dimensional projection plane reflecting the high-dimensional environmental information. In addition, it is easy to see, as shown in Section \ref{sec:Methodology}, that the new Riemannian metric on $(x_1,x_2)$ plane is isometric to the induced metric of submanifold by Euclidean metric on $\mathbb{R}^n$, so the general problem of optimal path planning in high-dimensional space is transformed into a geometric problem on the two-dimensional plane with new Riemannian metric. Thus, high-dimensional complex path planning problems can be solved on the two-dimensional plane. Furthermore, we apply the idea of Riemannian metric to the RRT* algorithm and present a series of simulation experiments. We test the performance of the RRT* algorithm based on Riemannian metric for optimal path planning in scenarios with different surface curvature and different dimensions.
Comparative experiments with the original RRT* algorithm using Euclidean distance are conducted. Under the condition that the number of sampling points and step size are the same, we find that the path retrieved by RRT*-R algorithm has better smoothness and optimization properties as the dimension of the workspace increases. 
Moreover, compared with the theoretical optimal path length, namely geodesic length, it is found that the difference between the path length retrieved by RRT*-R algorithm and the geodesic length is very small, and the robot effectively avoids the peak area where environmental factors such as height and ground resistance change dramatically, which verifies the accuracy of the algorithm. Finally, we conduct a large number of repeatability experiments and convergence experiments to verify the stability of the algorithm.

The rest of the paper is organized as follows. In Section \ref{sec:Problem Formulation and background}, we briefly introduce the notations in Riemannian geometry. A new Riemannian metric is constructed on the projection plane in Section \ref{sec:Methodology}. After that, we propose a Riemannian metric-based RRT*-R algorithm to solve the optimal path planning problem in Section \ref{sec:Algorithm}. To demonstrate the efficiency of the proposed method, experimental results are presented and discussed in Section \ref{sec:Simulation}. We conclude with a discussion of our method and directions for future work in Section \ref{sec:Conclusion}. In order to avoid disrupting the flow of the presentation, the detailed calculation process of Christoffel symbol and geodesic equation are presented in the Appendix.

\section{Preliminary}
\label{sec:Problem Formulation and background}

In this section, we introduce the necessary notations and definitions in Riemannian geometry \cite{boston1992riemannian} \cite{connecting2012thielhelm} for the formulation of our method, necessarily omitting some technicalities and details due to the lack of space. For notational compactness, we use the Einstein summation convention in this paper. It allows for the omission of the summation symbol $\sum $ and whenever an index variable appears twice in a term, it has to be added up for every possible value of the index.

We aim at constructing a new Riemannian metric on a Riemannian manifold. Commonly, a nonempty topological space $M$ is called m-dimensional topological manifold if is locally homeomorphic to the m-dimensional Euclidean space. A manifold together with an atlas is called differentiable of order $k$ if
the chart transitions are differentiable of order $k$ for each pair of charts. The collection of all tangent vectors to curves passing through a point $p$ forms a vector space $T_{p}M$, the tangent space of $M$ in $p$. Every chart $x$ covering $p$ induces a basis of $T_{p}M$ given by the vectors $\frac{\partial }{\partial x_i }  $ that are tangent to the
coordinate curves defined by $x$. 

Here we take the height surface in $\mathbb{R}^3$ as an example. Let the surface be represented by a smooth binary function $x_3: \mathbb{R}^2 \to \mathbb{R}$ via $X_h: (x_1, x_2)\mapsto (x_1, x_2, x_3(x_1, x_2))$ , where $x_1$ and $x_2$ are parameters on the plane. A height surface $M= X_h(\mathbb{R}^2)$ is obviously completely covered by the chart $X_h^{-1} $. The mathematical representation of the basis vectors of its tangent plane at each point can be determined by computing the partial derivatives of the surface. Let's consider a point $P$ on the surface with coordinates $(x_1^P, x_2^P, x_3^P)$. To compute the first basis vector, we fix $x_2 = x_2^P$ and make a small variation in $x_1$ (e.g. $x_1 = x_1^P + \varepsilon $), where $\varepsilon $ is an infinitesimal quantity. Then, we can calculate the corresponding change in the surface height $ \triangle x_3=x_3(x_1^P+\varepsilon,x_2^P)-x_3 (x_1^P, x_2^P)$. The components of the first basis vector can be expressed as  $\triangle x_1=\varepsilon$, $\triangle x_2=0$, and $\triangle x_3= \triangle x_3$. Therefore, the first basis vector is $ \vec{e_1}=(1,0,\frac{\partial x_3}{\partial x_1} )$. Similarly, the second basis vector is $\vec{e_2}=(0,1,\frac{\partial x_3}{\partial x_2} )$ as shown in Fig~\ref{tangentplane}. In this tangent plane, any vector can be represented as a linear combination of the basis vectors. Consider a vector in the tangent plane with components $(v_1, v_2, v_3)$, then it equals $$(v_1, v_2, v_3)= v_1\cdot (1,0,\frac{\partial x_3}{\partial x_1} )+ v_2\cdot (0,1,\frac{\partial x_3}{\partial x_2} ) $$
It's obvious that the third component $v_3= v_1\cdot \frac{\partial x_3}{\partial x_1}+v_2\cdot \frac{\partial x_3}{\partial x_2}$. 

\begin{figure}
    \centering
    \includegraphics[width=0.7\linewidth]{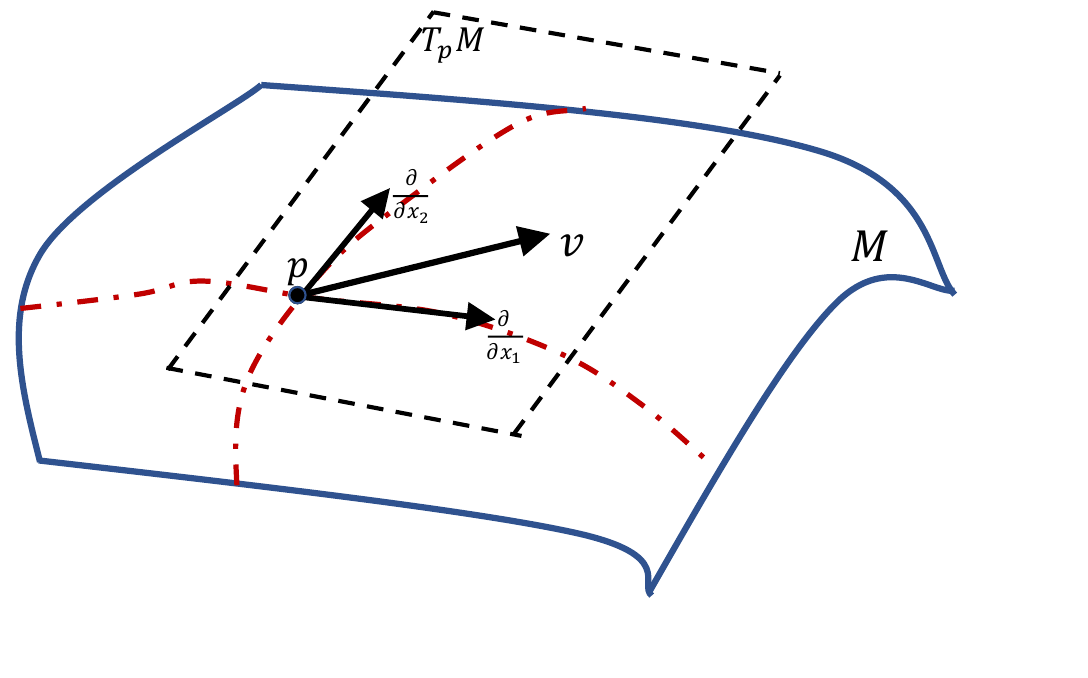}
    \caption{Tangent plane $T_{p}M$ at point $p$ on Riemannian manifold $M$ and its basis vector}
    \label{tangentplane}
\end{figure}

The Riemannian manifold $(M, g)$ consists of a $C^{\propto}$-manifold $M$ and a Riemannian metric $g_p$ which is a smooth second-order covariant tensor field on each of the tangent spaces $T_{p}M $ of $M$. Obviously, $p\mapsto g_p$ varies smoothly, which means that for any two smooth vector fields $X,Y$, the inner product $g_p(X|_p,Y|_p)$ is a smooth function of $P$. The subscript $P$ is usually omitted. It can be expressed in local coordinates as 
$$\left \langle v,w \right \rangle = g_{ij} v^i w^j  ,  g_{ij}=\left \langle \frac{\partial }{\partial x^i}, \frac{\partial }{\partial x^j} \right \rangle $$

\begin{definition}
The tensor $g$ is referred to as the Riemannian metric on $M$, if $g$ satisfies both symmetry and positive definiteness.

(1)symmetry: $g(u,v)=g(v,u)$, $\forall u,v\in T_{P}M $

(2)positive definiteness: $g(u,u)\ge$ 0 for all $ u\in$ $T_{p}M $, where the equal sign holds if and only if $u = 0$
\end{definition}

On a manifold $M$, we can multiply 1-forms to get bilinear forms: $ \theta_1 \cdot \theta_2 (v,w)=\theta_1 (v)\cdot \theta_2 (w) $, where $\theta_1 \cdot \theta_2 \ne  \theta_2\cdot \theta_1$ and the multiplication here is actually a tensor product $\theta_1 \cdot \theta_2=\theta_1\otimes \theta_2$. Further let $(U;x^i )$ denote an allowable local coordinate system where given coordinates $x(p)=(x^1,\cdots,x^m)$ on an open set U of M. Thus we can construct bilinear forms $dx^{i}\otimes dx^{j} $. Then, we can write $g|_U=g_{ij}dx^{i}\otimes dx^{j}$, here $g_{ij}=g(\frac{\partial }{\partial x^{i} },\frac{\partial }{\partial x^{j} } )$ and $g_{ij}=g_{ji}$. After the introduction of symmetric product $dx^i dx^j=\frac{1}{2} (dx^{i}\otimes dx^j+dx^j\otimes dx^i) $, we have the final expression $g|_U=g_{ij}dx^i dx^j$

We now turn our attention to the concept of isometric immersion. In essence, an isometric immersion is a smooth mapping that preserves the metric structure between two manifolds. This means that during the immersion process, geometric properties such as distances, angles, and curvatures remain unchanged between the two manifolds. It provides a powerful tool for studying the relationships between different manifolds and understanding their shared geometric properties.

\begin{definition}{(isometric immersion)}
An isometric immersion between Riemannian manifolds $(M,g)$ and $(N,h)$ is a smooth mapping $f: M\to N$ such that $g=f^{\ast } h$,i.e. $h(f_{\ast }(v), f_{\ast }(w))=g(v,w)$ for all tangent vectors $v,w\in T_{p}M $ and all $p\in M$
\end{definition}

The Riemannian metric induces a norm $\left \| v \right \| =\left \langle v,v \right \rangle ^{\frac{1}{2} } = \sqrt{g(v,v)} $ for tangent vectors. We can now define the length $L(\gamma )$ of a continuously differentiable curve $\gamma:[a, b] \rightarrow M$ via $$L(\gamma)=\int_{a}^{b}\left \| \dot{\gamma } (t) \right \| dt =\int_{a}^{b} \sqrt{g_{\gamma(t)}(\dot{\gamma}(t), \dot{\gamma}(t))} dt$$
Basing on the length definition, we introduce the distance function 
$$d_M (p,q)= \inf\left \{ L(\gamma )\mid \gamma \ connects \ p\  and\ q\right \} $$
which forms a metric space $d_M$. The computation of the precise distance between arbitrary points $p$ and $q$ remains a challenging task since the infimum mentioned above is typically hard to attain by analytic methods. However, a curve which locally minimizes the quantity $L$ satisfies a differential equation as shown in the next paragraph.

On a Riemannian manifold $M$, a curve $\gamma (t)$ is called a geodesic if its tangent vector along the curve has zero covariant derivative,satisfying the geodesic equation:
$$\nabla _{\dot{\gamma } (t)}\dot{\gamma } (t)=0$$
where $\dot{\gamma } (t)$ represents the tangent vector of the curve and $\nabla$ denotes the Levi-Civita connection on the Riemannian manifold $M$. The local existence and uniqueness theorem of geodesics states that for initial conditions $\gamma (t_0)=p$ and initial tangent vector $v_{0} $, the geodesic equation is a second-order nonlinear differential equation with a unique solution. In coordinate representation, the geodesic equation can be written as: 
$$\frac{\mathrm{d}^2 x^k}{\mathrm{~d} t^2}+\frac{\mathrm{d} x^{j}}{\mathrm{~d} t} \frac{\mathrm{d} x^i}{\mathrm{~d} t} \Gamma_{ji}^k=0, \quad 1 \leq k \leq m$$
where $x^{k}$ represents the coordinate functions on the manifold $M$, t is the parameter, and $\Gamma _{ij}^{k}$ are the Christoffel symbols, defined as:
$$\Gamma _{ij}^k =\frac{1}{2}g^{kl}\left(\frac{\partial g_{il} }{\partial x^j}+\frac{\partial g_{lj} }{\partial x^i}-\frac{\partial g_{ij} }{\partial x^l} \right)$$

Here $[g_{ij}]$ is the Riemannian metric tensor and $[g^{kl}]$ is its inverse matrix.



\section{Geometric model and method}
\label{sec:Methodology}
In this section, we take two-dimensional projection plane with a new Riemannian metric as a geometric model for solving the path planning problem.
This is due to the consideration that the constructed Riemannian metric contains environmental information in the high-dimensional Euclidean space, such as height change, ground resistance, etc., and the optimal path planning problem in high-dimensional workspace can be transformed into a geometric problem on the two-dimensional plane with new Riemannian metric.
\subsection{Path planning on surface in $\mathbb{R}^{3}$}
Different from the two-dimensional plane on which the robot can move freely, the surface in three-dimensional space is a three-dimensional object with curvature and local geometric characteristics. In order to maintain contact with the surface or follow specific motion constraints, the robot needs to consider both the tangent plane and the direction of the surface normal vector during motion planning, so as to determine the appropriate path and adjust the moving direction.
In addition, the measurement of motion length is complicated in the motion planning of three-dimensional space surfaces. For a two-dimensional plane, the length of motion can simply be measured using linear distance or curve length. However, in three-dimensional space, the geometric properties of the surface cause the path to bend and extend in three-dimensional space, and the linear distance does not accurately reflect the actual motion distance on the surface, and the length of the motion on the surface usually requires the use of more complex measures.

According to the derivation of the basis vectors in the tangent plane of the height surface in Section \ref{sec:Problem Formulation and background}, it can be clearly seen that locally, when a point on the surface moves by unit length in the direction of $x_1 $, it will bring increment $\frac{\partial x_3 }{\partial x_1 } $ in the direction of $x_3 $. Similarly, when a point moves by unit length in the direction of $x_2 $, it will bring increment $\frac{\partial x_3 }{\partial x_2 } $ in the direction of $x_3 $. 
Our construction of the isometric Riemannian metric on the projection plane is inspired by it, the details of which are presented in the generalized higher dimensional case in the next section. 
Moreover, when the movement of the robot on the surface is also affected by factors such as ground resistance, the above discussion on the surface height function $x_3 $ is still applicable to the ground resistance function $x_4 $ .

\subsection{Path planning on surface in $\mathbb{R}^n$}
In order to solve the path planning problem, we propose the following geometric model. In this section, we present a general framework for optimal path planning by constructing a new Riemannian metric on the projection plane $\mathbb{R}^2 $ from the standard Euclidean metric on the two-dimensional smooth manifold in $\mathbb{R}^n$.

Denote the two-dimensional smooth manifold $M$ in n-dimensional space as $$\begin{aligned}
\vec{s}&: U \subset\mathbb{R}^2 \rightarrow \mathbb{R}^n \\
\vec{s}\left(x_1, x_2\right)&=\left(x_1, x_2, x_3\left(x_1, x_2\right), \cdots, x_n\left(x_1, x_2\right)\right)
\end{aligned}$$
where $U$ is a open set in $\mathbb{R}^2$. Obviously, $\vec{s}\left( U \right)$ is embedded in the n-dimensional space $\mathbb{R}^n$,where the standard Euclidean metric can be denoted as $g_{ij}$. The following is about its calculation: 
$$\begin{aligned}
g_{11}&=g((1,0,\cdots,0),(1,0,\cdots,0))=1 \\
g_{12}&=g((1,0,\cdots,0),(0,1,\cdots,0))=0
\end{aligned}$$
Similarly, 
$g_{ii}=1 (i=1,\cdots,n) $ and $g_{ij}=0 (i\ne j)$.

\noindent Finally we have  
$g=g_{i j} d x^{i} \otimes d x^j=g_{ij} d x^i d x^j =g_{11} d x^1 d x^1+g_{22} d x^2 d x^2+\cdots+g_{nn} d x^n d x^n = d x^1 d x^1+d x^2 d x^2+\cdots+d x^n d x^n $. Choose two vectors on the tangent plane $T_{p}M $ at any point $p$ on the manifold $\vec{s}(U)$ as follows:
$$\begin{aligned}
&\left(1,0, \frac{\partial x_3}{\partial x_1}, \frac{\partial x_4}{\partial x_1}, \ldots, \frac{\partial x_n}{\partial x_1}\right) \ldots . (\ast ) , \\
&\left(0,1, \frac{\partial x_3}{\partial x_2}, \frac{\partial x_4}{\partial x_2}, \ldots, \frac{\partial x_n}{\partial x_2}\right) \ldots (\ast\ast )
\end{aligned}$$
Now we can construct a new Riemannian metric $h=h_{i j} d x^i d x^j=h_{11} d x^1 d x^1 + h_{12} d x^1 d x^2 + h_{21} d x^2 d x^1 + h_{22} d x^2 d x^2$ on the projection plane $\mathbb{R}^2 $.
Let 
$$
\begin{aligned}
h_{11}&=g((\ast ), (\ast ))=1+\sum\limits_{k=3}^n{\left(\frac{\partial x_k}{\partial x_1}\right)^2}\\
h_{12}&=h_{21}=g( (\ast), (\ast\ast) )=\sum\limits_{k=3}^n\frac{\partial x_{k}}{\partial x_{1}} \cdot \frac{\partial x_k}{\partial x_2}\\
h_{22}&=g( (\ast\ast), (\ast\ast) )=1+\sum\limits_{k=3}^n\left(\frac{\partial x_k}{\partial x_2}\right)^2
\end{aligned}
$$

\noindent Thus the Metric matrix $[h_{ij}]$ can be written as 
$$
\large \left[h_{ij}\right]=\begin{bmatrix}
    1+\sum\limits_{k=3}^n\left(\frac{\partial x_k}{\partial x_1}\right)^2 & \sum\limits_{k=3}^n\frac{\partial x_k}{\partial x_1} \cdot \frac{\partial x_k}{\partial x_2}\\
    \sum\limits_{k=3}^n\frac{\partial x_k}{\partial x_1} \cdot \frac{\partial x_k}{\partial x_2} & 1+\sum\limits_{k=3}^n\left(\frac{\partial x_k}{\partial x_2}\right)^2 
\end{bmatrix}
$$

Furthermore, the inverse matrix can be calculated by
$\left[h_{i j}\right]^{-1} \triangleq\left[h^{k l}\right]$

$${\large \left[h^{k l}\right]=\frac{1}{d}\begin{bmatrix}
1+\sum\limits_{k=3}^n\left(\frac{\partial x_k}{\partial x_2}\right)^2&-\sum\limits_{k=3}^n \frac{\partial x_k}{\partial x_1} \cdot \frac{\partial x_k}{\partial x_2} \\
-\sum\limits_{k=3}^n \frac{\partial x_k}{\partial x_1} \cdot \frac{\partial x_k}{\partial x_2} & 1+\sum\limits_{k=3}^n\left(\frac{\partial x_k}{\partial x_1}\right)^2
\end{bmatrix}
} $$
here
$$
\begin{aligned}
  d= \det(h_{ij})=&\left[1+\sum\limits_{k=3}^n\left(\frac{\partial x_k}{\partial x_1}\right)^2\right] \cdot \left[1+\sum\limits_{k=3}^n\left(\frac{\partial x_k}{\partial x_2}\right)^2\right] -\left[\sum\limits_{k=3}^{n} \frac{\partial x_k}{\partial x_1} \cdot \frac{\partial x_k}{\partial x_2}\right]^{2}  
\end{aligned}  
$$

We also point out that the projection from this two-dimensional smooth manifold $\vec{s}$ in n-dimensional space to plane $\mathbb{R}^2$ is an isometric map.
\begin{theorem}\label{thm2}
A two-dimensional smooth manifold $M$ in $\mathbb{R}^n$ is equipped with the standard Euclidean metric $g=d x^1 d x^1+d x^2 d x^2+\cdots+d x^n d x^n$. If its projection plane $\mathbb{R}^2$ is equipped with a new Riemannian metric $h=h_{i j} d x^i d x^j$ as above, then the Projection mapping $f: M \rightarrow \mathbb{R}^2$ is an isometric immersion.
\end{theorem}
\begin{proof}
    We showed in Definition 2 that a diffeomorphism is isometric if it satisfies $g=f^{\ast } h$. Hence, it remains to be shown that $h(f_{\ast }(v), f_{\ast }(w))=g(v,w)$ for all tangent vectors $v,w\in T_{p}M $ and all $p\in M$. Let $v=(v_1,v_2,v_3,\cdots,v_n),w=(w_1,w_2,w_3,\cdots,w_n)$. The left side equals (1)

\begin{align}
    &h(f_{\ast }(v), f_{\ast }(w))=  h\left(\left(v_{1}, v_{2}\right),\left(w_1, w_2\right)\right) \notag\\
    &=  h_{11} v_1 w_1+h_{12} v_1 w_2+h_{21} v_2 w_1+h_{22} v_2 w_2 \notag\\
    &=  v_1 w_1\left[1+\sum\limits_{k=3}^n\left(\frac{\partial x_k}{\partial x_1}\right)^2\right] + v_2 w_2 \left[1+\sum\limits_{k=3}^n\left(\frac{\partial x_k}{\partial x_2}\right)^2\right]\notag\\ &+\left(v_1 w_2+v_2 w_1\right) \left[\sum\limits_{k=3}^n \frac{\partial x_k}{\partial x_1}\cdot\frac{\partial x_k}{\partial x_2}\right] 
\end{align}

Since 
$\left(v_1, v_2, v_3,\cdots,v_n\right)=\left(v_1, 0, v_1 \frac{\partial x_3}{\partial x_1},\cdots, v_1 \frac{\partial x_n}{\partial x_1} \right)+\left(0, v_2, v_2\frac{\partial x_3}{\partial x_2},\cdots,v_2 \frac{\partial x_n}{\partial x_2}\right)$, then $v_n$ has an expression $v_n=v_1\frac{\partial x_n}{\partial x_1}+v_2 \frac{\partial x_n}{\partial x_2}$. Similarly, $w_n=w_1\frac{\partial x_n}{\partial x_1}+w_2 \frac{\partial x_n}{\partial x_2}$. Thus the right side equals (2)

\begin{align}
& g(v,w)= v_1 w_1+v_2 w_2+v_3 w_3+\cdots+v_n w_n \notag\\
& = v_1w_1\left[1+\sum\limits_{k=3}^n\left(\frac{\partial x_k}{\partial x_1}\right)^2\right]  
+ v_2 w_2\left[1+\sum\limits_{k=3}^n\left(\frac{\partial x_k}{\partial x_2}\right)^2\right]\notag \\
&+\left(v_1 w_2 + v_2 w_1\right)\left[\sum\limits_{k=3}^n \frac{\partial x_k}{\partial x_1}\cdot\frac{\partial x_k}{\partial x_2}\right]
\end{align} 

We see the left side equals the right side and this completes the proof.
\end{proof}

Isometry is a very important in our theory. After proving that projection mapping maintains the isometric properties of the metric, we will naturally find that the length of a curve on a two-dimensional smooth manifold is equal to the length of a new curve on its projection plane in the new metric sense. This is very important in engineering practice, which means that we transform the problem of finding the shortest path in $\mathbb{R}^n$ into a new problem of shortest path retrieval on the $\mathbb{R}^2$ plane. In practice, the reduction of dimension often leads to more convenient calculation and shorter response time, which will be more obviously reflected in the case of higher dimensions. 

Using the proposed Riemannian metric $h=h_{i j} dx^i dx^j$ , the length (cost function) of a smooth curve $\gamma:[a, b] \rightarrow \mathbb{R}^2$ on projection plane $\mathbb{R}^2 $ is obtained by the integral over inner products on $\mathbb{R}^2 $ as shown in (3).

\begin{align}
\hspace{-5cm} 
L(\gamma) = &\int_a^b \left|\gamma^{\prime}(t)\right| dt=\int_a^b \sqrt{h\left(\gamma^{\prime}(t), \gamma^{\prime}(t)\right)} dt \notag\\
= &\int_a^b \sqrt{h\left(\left(x_1^{\prime}(t), x_2^{\prime}(t)\right),\left(x_1^{\prime}(t), x_2^{\prime}(t)\right)\right)} dt \notag\\
= &\int_a^b \sqrt{h_{11}\left[x_1^{\prime}(t)\right]^2+h_{12}\left[x_1^{\prime}(t)\cdot  x_2^{\prime}(t)\right]+h_{21}\left[x_2^{\prime}(t) \cdot x_1^{\prime}(t)\right]+h_{22}\left[x_2^{\prime}(t)\right]^2} dt \notag\\
= &\int_a^b \sqrt{h_{11}\left[x_1^{\prime}(t)\right]^2+2 h_{12}\left[x_1^{\prime}(t) \cdot x_2^{\prime}(t)\right]+h_{22}\left[x_2^{\prime}(t)\right]^2} dt \notag\\
= &\int_a^b \sqrt{\left[1+\sum_{k=3}^n\left(\frac{\partial x_k}{\partial x_1}\right)^2\right] \cdot\left[x_1^{\prime}(t)\right]^2+2\left[ \sum_{k=3}^n\frac{\partial x_k}{\partial x_1}\cdot\frac{\partial x_k}{\partial x_2}\right] \cdot\left[x_1^{\prime}(t)\cdot x_2^{\prime}(t)\right]}\notag\\
&\overline{+\left[1+\sum_{k=3}^n\left(\frac{\partial x_k}{\partial x_2}\right)^2\right] \cdot\left[x_2^{\prime}(t)\right]^2 }dt
\end{align}

Compared with the standard Euclidean distance, the curve length is greater in most cases and is affected by multiple smooth functions $x_k\left(x_1, x_2\right),$ $k=3,\cdots,n$. Specifically, function $x_3\left(x_1, x_2\right)$ describes the relationship between the height of a point on a surface and the $x_1$ and $x_2$ coordinates of the corresponding point. Then according to the above curve length formula, the length on projection plane increases as the curvature of the corresponding two-dimensional surface increases. More intuitively, the greater the “slope” of the surface, the longer the distance of the corresponding region on the projection plane.

Now we have constructed a new Riemannian metric on the two-dimensional plane, and transformed the motion planning problem of higher dimensions into a motion planning problem of two-dimensional plane in the sense of the new metric. Therefore, some existing classical path planning algorithms, such as Dijkstra's algorithm, A* algorithm, RRT algorithm, etc., can be used in this two-dimensional plane with new Riemannian metric. In the next section, we take the RRT* algorithm as an example to show how this new Riemannian metric is specifically applied to the path planning algorithm.

\section{Algorithm}
\label{sec:Algorithm}
In this section, an incremental sampling-based algorithm RRT*-R is proposed as follows, which is designed to solve the Riemannian metric-based path planning problem.

As an incremental motion planning algorithm, asymptotic optimality of RRT* is ensured by the following Theorem \cite{firstRRT*2011sampling}. 
\begin{theorem}\label{thm3}
Let $\textup{cost}_i^{RRT*}$ be a variable that denotes the minimum cost of the path found by RRT* at the end of iteration $i$, here the number of sampled points equals to $i$. If the cost function is additive and continuous, and the space satisfies some assumptions that each point has a ball neighborhood, then the cost of the minimum cost path found by RRT* converges to minimum cost $c^*$ almost surely, i.e.
$$\begin{aligned}
    P\left(\{\lim\limits_{i\rightarrow +\infty}\textup{cost}_i^{RRT*}=c^*\}\right)=1
\end{aligned}$$
\end{theorem}

It's easy to verify that the cost function and configuration space satisfy the conditions mentioned above, so the tasks can be processed by the proposed RRT*-R. Unlike the common cost function given by Euclidean distance, in the proposed Riemannian-based RRT*-R algorithm, the cost function between $v_1$ and $v_2$ is replaced by $\textup{Line-R}(v_1, v_2)$, where $v_1=(x^1_1,x^1_2), v_2=(x^2_1,x^2_2)\in C_{free} \subset \mathbb{R}^2$. In fact, $\textup{Line-R}(v_1, v_2)$ is the length of path $r(t), t\in [0, 1]$ in ${\mathbb{R}^2}$ with new Riemannian metric $h$. Here $r(t)$ is defined as:
\begin{align*}
    r(t) &= (x_1(t), x_2(t), x_3(t), \cdots, x_n(t)),\; t\in[0, 1]\\
    x_i(t) &= t * x^2_i + (1-t) * x^1_i, \; i=\{1, 2\}\\
    x_j(t) &= x_j(x_1(t), x_2(t)),\; j\geq 3
\end{align*}

Since the length on ${\mathbb{R}^2}$ with new Riemannian metric $h$ is calculated by an integration, it's necessary to choose an appropriate numerical integration method for the purpose of getting high integration accuracy and acceptable computing speed. Here we use Newton-Cotes integration method, which is a fifth-order algebraic accuracy algorithm, to get the numerical result of $\int_0^1|\frac{dr(t)}{dt}|dt$ as long as the edge connecting vertices $v_1$ and $v_2$ doesn't traverse the obstacle space $C_{obs}$ (see Alg. 1).

Similar to the standard RRT* algorithm \cite{firstRRT*2011sampling}, the proposed RRT*-R algorithm consists of a main procedure (see Alg.2) and its \textup{Extend} procedure (see Alg.3 ). The main idea of RRT*-R is to construct a tree incrementally until it reaches the goal region. The difference is that in the proposed RRT*-R algorithm, the length in the sense of new Riemannian metric is chosen as the cost instead of Euclidean length. Since the new Riemannian metric contains environmental information, the Riemannian metric-based RRT*-R algorithm performs differently from the standard Euclidean metric-based RRT* algorithm.

\begin{algorithm}
\caption{$\textup{Line-R}(v_1=(x^1_1,x^1_2), v_2=(x^2_1,x^2_2))$.}\label{alg:alg1}
\begin{algorithmic}
\STATE 
\STATE 
$\hat{x^k_i}=\frac{k}{4}x^2_i+\frac{4-k}{4}x^1_i, k\in\{0,1,2,3,4\}, i\in\{1,2\};$
\STATE  
$\hat{v}_k=\left(\hat{x^k_1},\hat{x^k_2},x_3(\hat{x^k_1},\hat{x^k_2}),\cdots,x_m(\hat{x^k_1},\hat{x^k_2})\right),k\in$\\
$\{0, 1, 2, 3, 4\};$
\STATE  $r_k=\sqrt{\sum\limits_{i=1}^{2}\sum\limits_{j=1}^{2} h_{ij}(\hat{x^k_1},\hat{x^k_2}) (x^2_i-x^1_i)(x^2_j-x^1_j)},$\\
$k\in\{0,1,2,3,4\};$
\STATE  
$s=\frac{7}{90}r_0+\frac{16}{45}r_1+\frac{2}{15}r_2+\frac{16}{45}r_3+\frac{7}{90}r_4;$
\STATE \textbf{return} $s$
\end{algorithmic}
\label{alg1}
\end{algorithm}

\begin{algorithm}
\caption{Body of RRT*-R.}\label{alg:alg2}
\begin{algorithmic}
\STATE 
\STATE  
$V \gets \{ v_{\textup{init}} \};E \gets \varnothing;i\gets 0;$ 
\WHILE{$i < N$}
\STATE  
$G\gets (V, E);v_{\textup{rand}} \gets \textup{Sample{(i)}}; i \gets i+1;$
\STATE  
$(V,E) \gets \textup{Extend}(G,v_{\textup{rand}})$
\ENDWHILE
\end{algorithmic}
\label{alg2}
\end{algorithm}

\begin{algorithm}
\caption{$\textup{Extend}_{RRT*-R}(G,v)$.}\label{alg:alg3}
\begin{algorithmic}
\STATE 
\STATE  $V^{\prime} \gets V;E^{\prime} \gets E;$ 
\STATE $v_{\textup{nearest}} \gets \textup{Nearest}(G, v);$
\STATE $v_{\textup{new}} \gets \textup{Steer}(v_{\textup{nearest}}, v)$
\IF{$\textup{ObstacleFree}(v_{\textup{nearest}},v_{\textup{new}})$}
\STATE  $V^{\prime} \gets V^{\prime}\cup 
        \{v_{\textup{new}}\};v_{\textup{min}} \gets v_{\textup{nearest}};$
\STATE $V_{\textup{near}} \gets \textup{Near}(G, v_{\textup{new}},|V|);$
\FOR{all $v_{\textup{near}}\in V_{\textup{near}}$}
\IF{$\textup{ObstacleFree}(v_{\textup{near}},v_{\textup{new}})$}
\STATE $c^{\prime}\gets \textup{Cost}(v_{\textup{near}})+\textup{Line-R}(v_{\textup{near}},v_{\textup{new}});$ 
\IF{$c^{\prime}<\textup{Cost}(v_\textup{new})$}
\STATE $v_{\textup{min}}\gets v_{\textup{near}}$
\ENDIF
\ENDIF
\ENDFOR
\STATE $E^{\prime} \gets E^{\prime}\cup \{(v_{\textup{min}},v_{\textup{new}})\};$
\FOR{all $v_{\textup{near}}\in V_{\textup{near}}\backslash\{v_{\textup{min}}\}$}
\IF{$\textup{ObstacleFree}(v_{\textup{new}},v_{\textup{near}})$ and $\textup{Cost}(v_{\textup{near}})$\\
$>\textup{Cost}(v_{\textup{new}})+\textup{Line-R}(v_{\textup{new}},v_{\textup{near}})$}
\STATE $v_{\textup{parent}}\gets \textup{Parent}(v_{\textup{near}});$
\STATE $E^{\prime} \gets E^{\prime}\backslash \{(v_{\textup{parent}},v_{\textup{near}})\};E^{\prime} \gets E^{\prime}\cup \{(v_{\textup{new}},v_{\textup{near}})\};$
\ENDIF
\ENDFOR
\ENDIF
\STATE \textbf{return} $G^{\prime}=(V^{\prime},E^{\prime})$
\end{algorithmic}
\label{alg3}
\end{algorithm}

\section{Simulation}
\label{sec:Simulation}
This section is devoted to an experimental study of the algorithms. Examples of two-dimensional surfaces in three-dimensional and four-dimensional spaces are considered. Comparison experiments with the original RRT* algorithm using Euclidean distances in high-dimensional spaces are included. Moreover, RRT*-R algorithms and the geodesic are compared with respect to their length of the solution achieved. The algorithms are run in high dimensional parameter space, which are essentially 2-dimensional surfaces embedded in $\mathbb{R}^n$. The RRT*-R algorithms are run in a square environment $[-1,11]\times [-1,11]$ with start point $(0,0)$ and end point $(10,10)$. All the algorithms were implemented in Python 3.10 and run on a computer with 3.2 GHz AMD Ryzen7 processor and 32GB DDR4 RAM running the Windows operating system.

In order to verify the accuracy of the RRT*-R algorithm in retrieving the path on the projection plane $(\mathbb{R}^2,h )$ from another perspective, we take the shortest geodetic length between the starting point and the ending point on the projection plane $(\mathbb{R}^2,h)$ as its reference. The derivation and calculation process of geodesic equations are shown in Appendix. Since the solution of the geodesic equation is determined by the initial position and the initial velocity direction, by traversing more than 200 geodesics (green lines) with different initial velocity directions within the angle range of $[0,\frac{\pi}{2}]$, we select the geodesics (blue lines) that passed through the neighborhood with a radius of 0.5 near the end point, and then find the geodesic with the shortest length (red line) in the set of blue geodesics, for example, Figure~\ref{3d1-3}. The length of the red geodesic on the projection plane $(\mathbb{R}^2,h)$ is the theoretical shortest route length between two points in robot path planning problem.
Besides the geodesic line, we also provide the original RRT* algorithm using Euclidean distance as a comparative experiment of the RRT*-R algorithm in terms of running performance and optimality.

\subsection{Experiments in three-dimensional space}
In the first 3-D application scenario, for the convenience of simulation, we choose the function $x_3 = 5\cdot e^{\left [ -\frac{1}{10}(x_1-5)^2-\frac{1}{10}(x_2-5)^2   \right ] } $ with Gaussian distribution shape to simulate the “hill” in three-dimensional space. In Figure~\ref{3d1peak}, the trees maintained by RRT*-R (Figure~\ref{3d1-2}) and its corresponding shortest route in 3-D space (Figure~\ref{3d1-1}) are shown respectively. It can be observed that the route chosen by RRT*-R with a length of 16.19926 (Figure~\ref{3d1-2}) is only about 0.05 different from the shortest geodesic with a length of 16.14253 (Figure~\ref{3d1-3}) obtained after 200 traversals on the projection plane $(\mathbb{R}^2,h)$. With the same number of sampling points and the same step size as in the RRT*-R algorithm, the path length retrieved by the original RRT* algorithm using Euclidean distance on the surface in 3-D Euclidean space is 16.5047, as shown in figure~\ref{3d1-4}. Moreover, it can be seen that the geometric shapes of the three curves are also almost the same. In order to increase the difficulty of retrieving paths by RRT*-R algorithm, we construct a “three-peak surface” (Figure~\ref{3d3-1}) similar to the “single-peak surface” above. The surface function is $x_3 = 7\cdot e^{\left [ -\frac{1}{2}(x_1-3)^2-\frac{1}{2}(x_2-3)^2  \right ] }+6\cdot e^{\left [ -\frac{1}{2}(x_1-7)^2-\frac{1}{2}(x_2-3)^2   \right ] }+5\cdot e^{\left [ -\frac{1}{2}(x_1-5)^2-\frac{1}{2}(x_2-7)^2   \right ] } $. In this case, The RRT*-R algorithm choose a route (Figure~\ref{3d3-2}~\ref{3d3-1}) with a length of 15.44211, which is only about 0.07 larger than the shortest blue geodesic (Figure~\ref{3d3-3}) with a length of 15.37167, obtained by traversing 400 times. The path (Figure~\ref{3d3-4}) length retrieved on the surface by the original RRT* algorithm using Euclidean distance is 16.1702, which is 0.73 longer than the path obtained by the RRT*-R algorithm. Obviously, when the RRT*-R algorithm retrieves a path on the projection plane $\mathbb{R}^2 $ that we constructed with the new Riemannian metric $h$, it chooses to advance in the low places between the mountains rather than over the peaks.

\begin{figure*}
   \centering
   \subfloat[]{
    \includegraphics[width=0.25\linewidth]{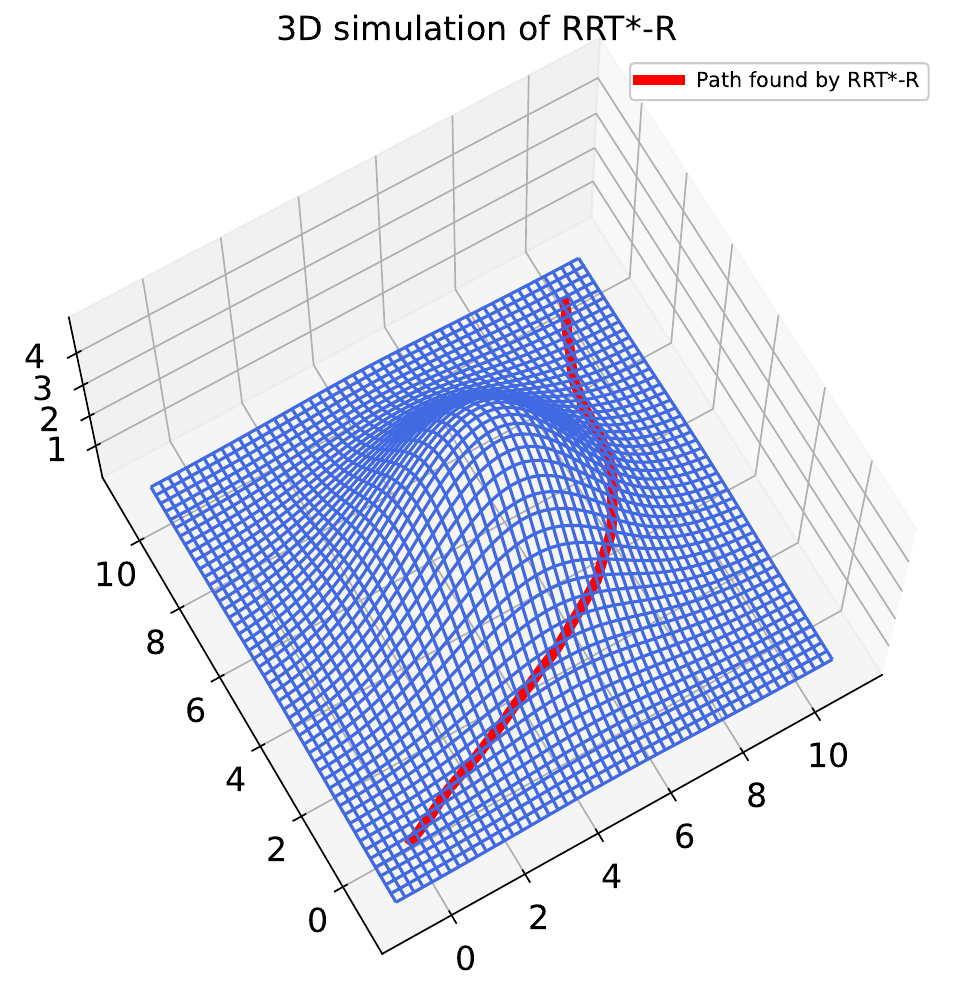}
    \label{3d1-1}
    }
   \subfloat[]{
      \includegraphics[width=0.25\linewidth]{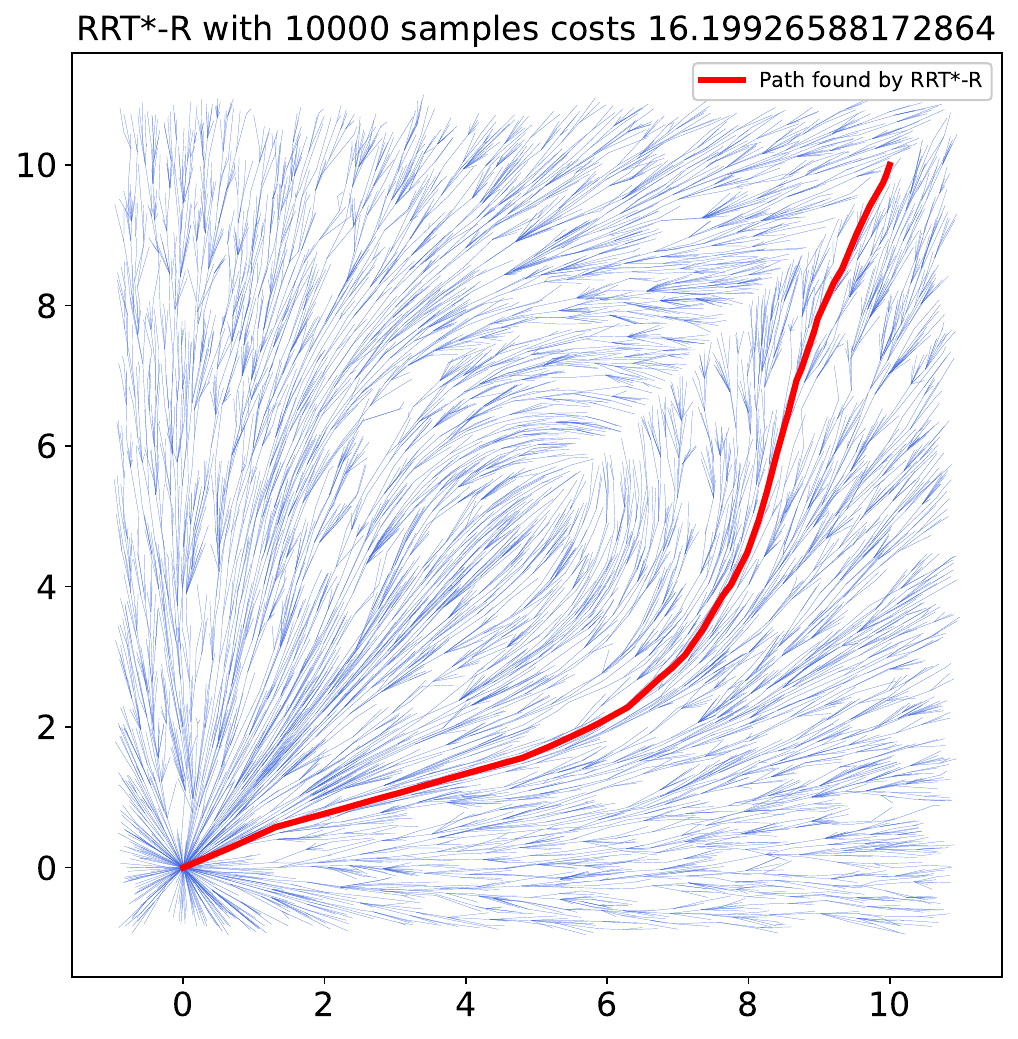}
      \label{3d1-2}
    }
    \subfloat[]{
      \includegraphics[width=0.25\linewidth]{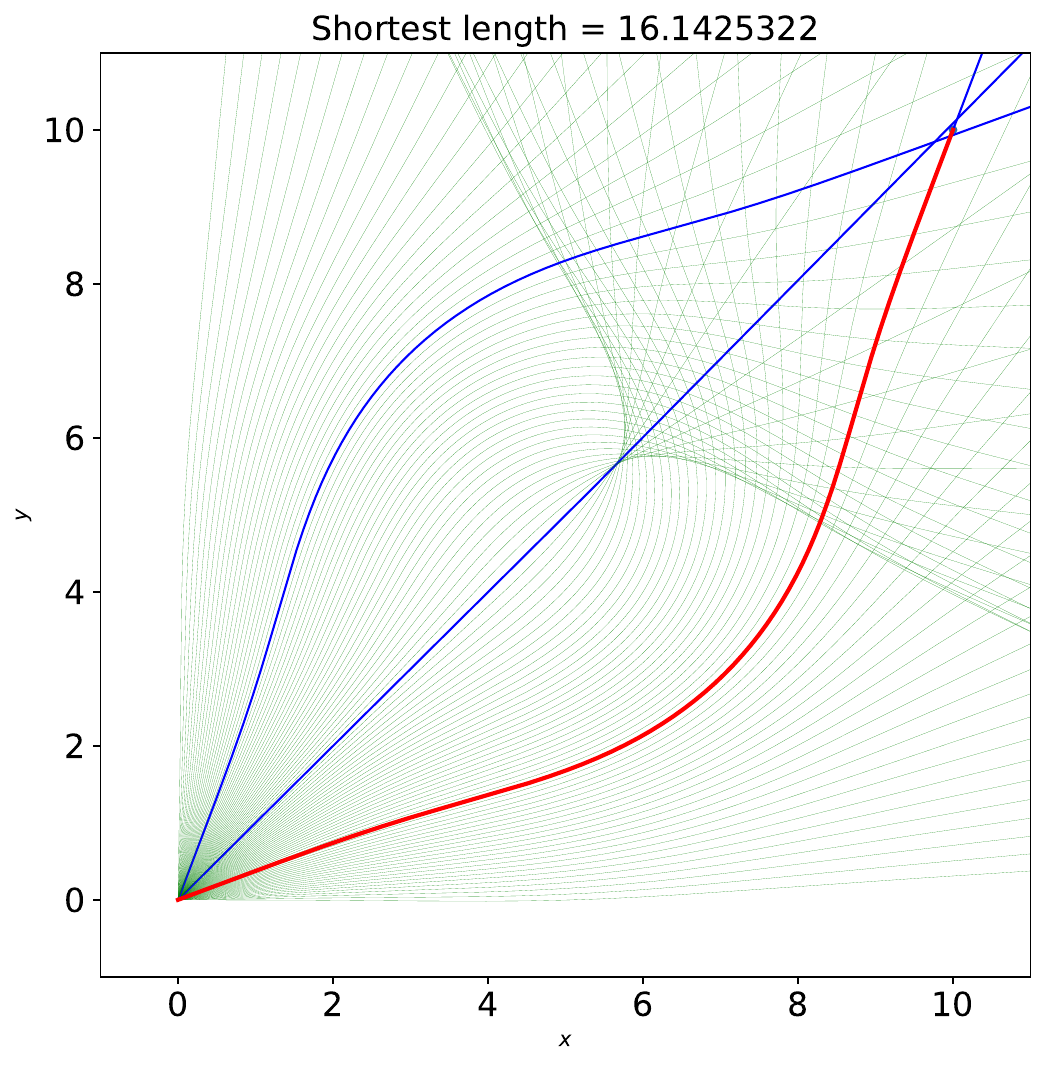}
      \label{3d1-3}
    }
    \subfloat[]{
      \includegraphics[width=0.25\linewidth]{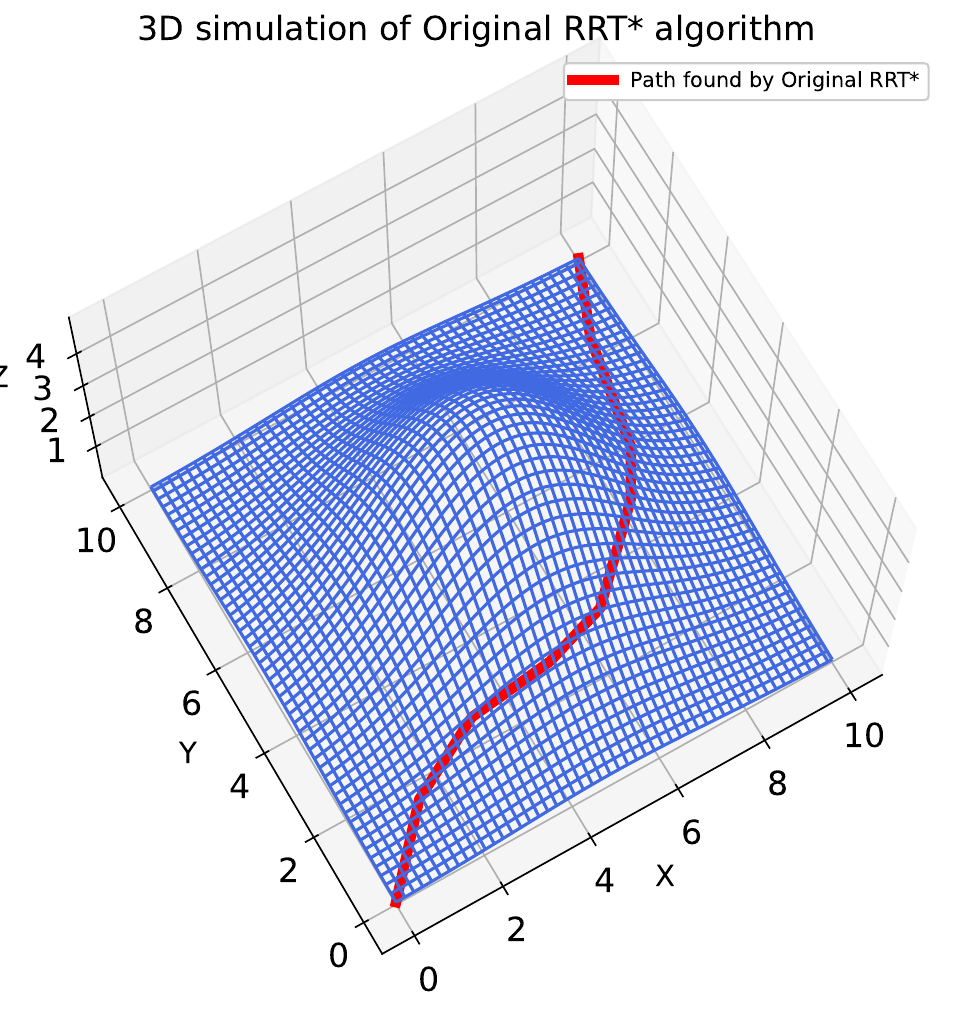}
      \label{3d1-4}
    }
   \captionsetup{justification=centering} 
   \caption{Simulation with 3-D scenario (one peak). (b) is obtained by RRT*-R taking 10000 samples. (a) is the preimage of the path in (b) under the projection map. (c) is geodesic path. (d) is obtained by original RRT* algorithm using Euclidean distance}
   \label{3d1peak}
\end{figure*}

\begin{figure*}
   \centering
   \subfloat[]{
    \includegraphics[width=0.25\linewidth]{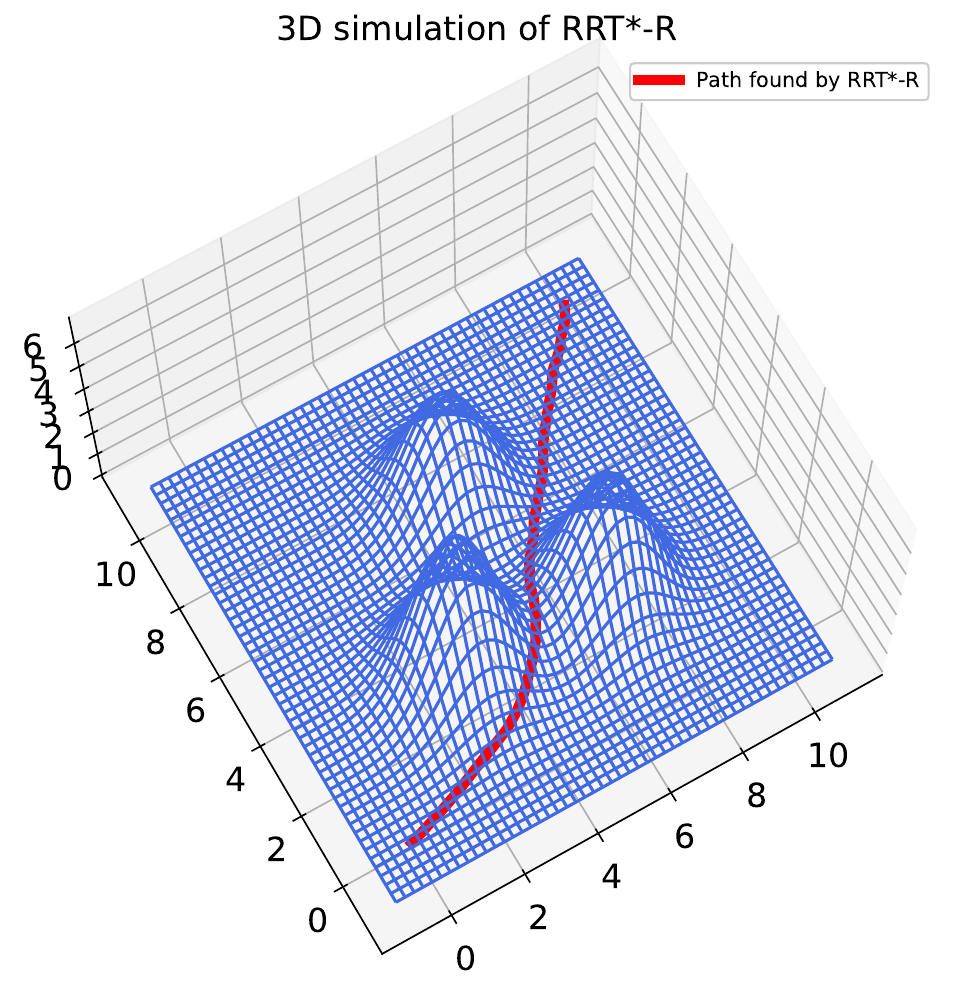}
    \label{3d3-1}
    }
   \subfloat[]{
      \includegraphics[width=0.25\linewidth]{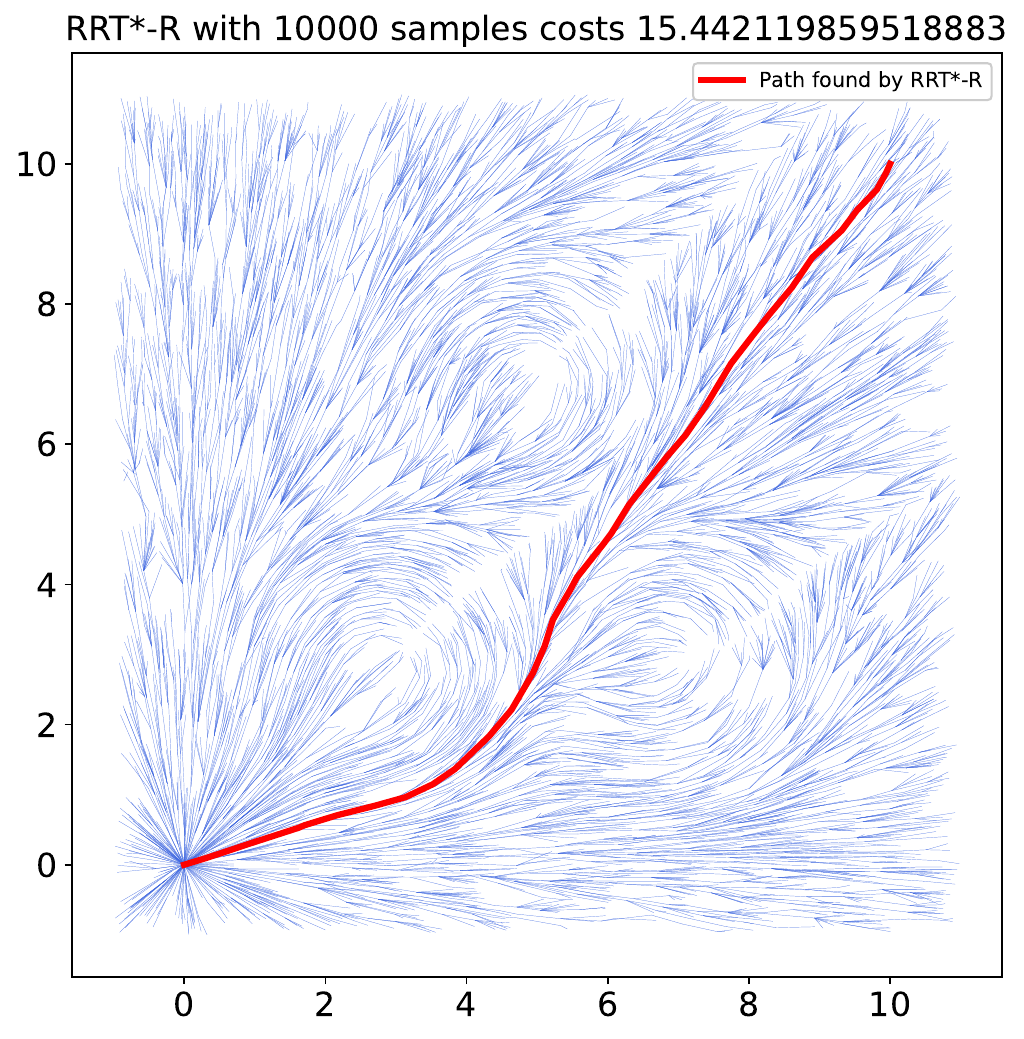}
      \label{3d3-2}
    }
    \subfloat[]{
      \includegraphics[width=0.25\linewidth]{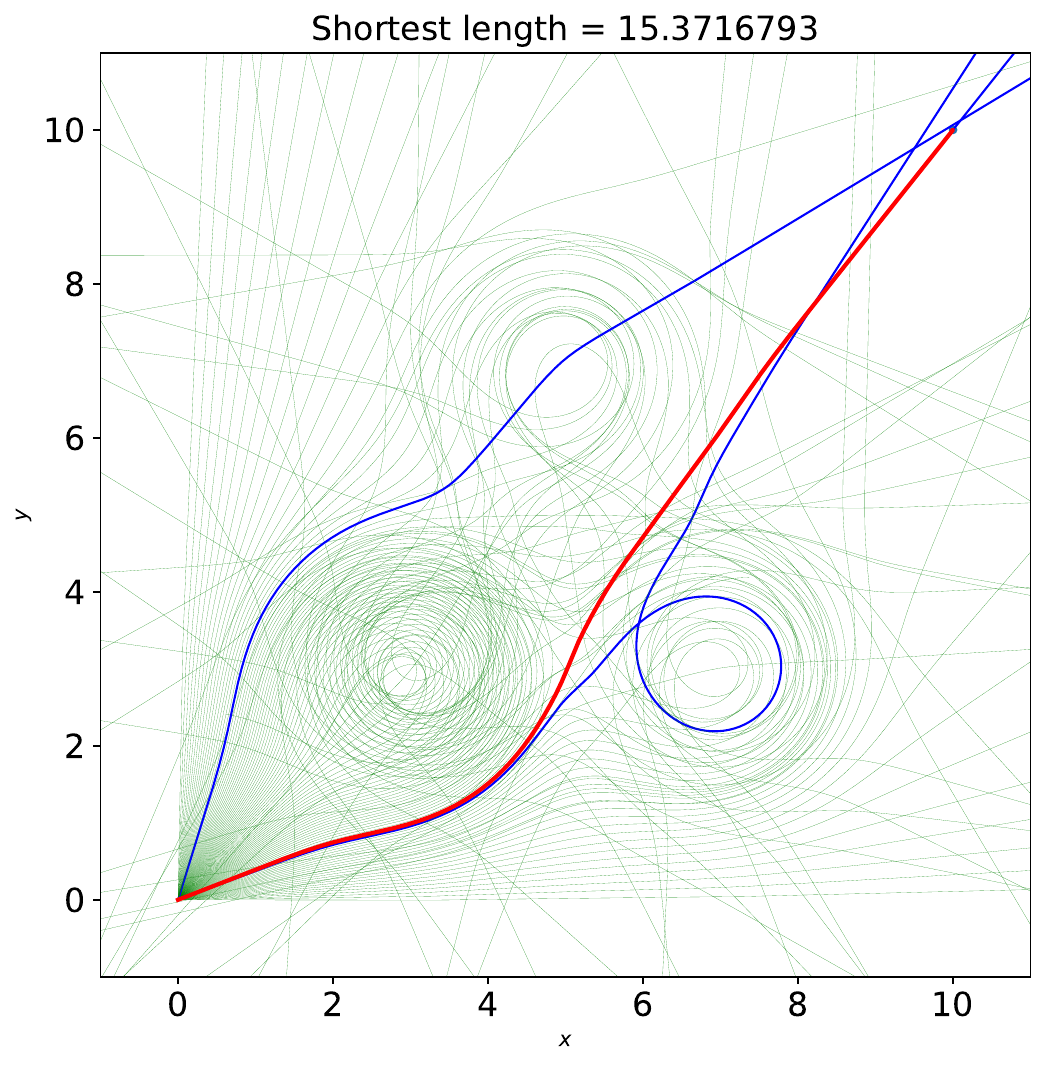}
      \label{3d3-3}
    }
    \subfloat[]{
      \includegraphics[width=0.25\linewidth]{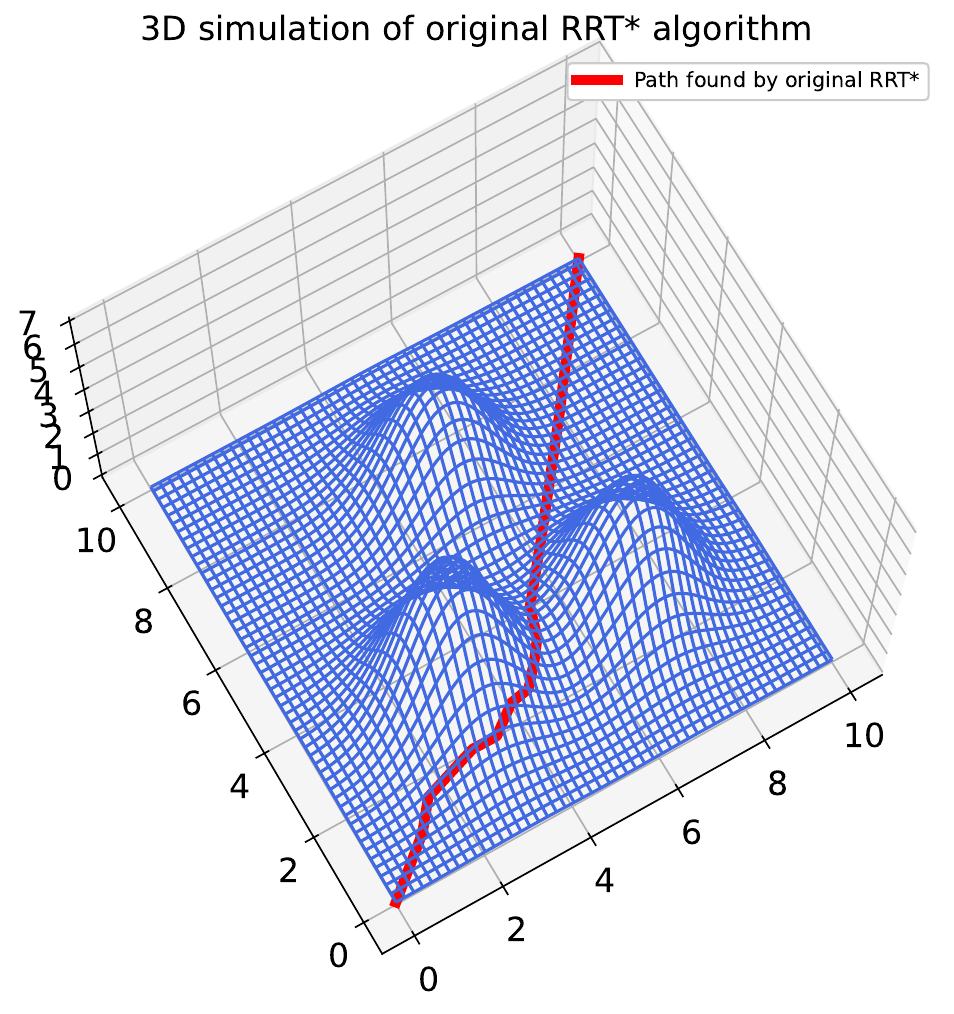}
      \label{3d3-4}
    }
   \captionsetup{justification=centering} 
   \caption{Simulation with 3-D scenario (three peaks). (b) is obtained by RRT*-R taking 10000 samples. (a) is the preimage of the path in (b) under the projection map. (c) is geodesic path. (d) is obtained by original RRT* algorithm using Euclidean distance}
   \label{3d3peak}
\end{figure*}

\subsection{Experiments in four-dimensional space}
In the second 4-D scenario, the impact of changes in ground friction resistance is also taken into consideration, which forms a path planning problem on a two-dimensional surface in four-dimensional workspace. When a robot moves on the ground, switching between different ground materials (such as soil, sand, gravel, grass, concrete, etc.) will affect the friction force on the robot, thereby affecting energy consumption. We use the concept of “ground resistance function” to describe the influence of the above ground factors on the robot's moving resistance. We choose $x_3 = 5\cdot e^{\left [ -\frac{1}{2}(x_1-3)^2-\frac{1}{2}(x_2-3)^2   \right ] }+5\cdot e^{\left [ -\frac{1}{2}(x_1-7)^2-\frac{1}{2}(x_2-3)^2   \right ] }+5\cdot e^{\left [ -\frac{1}{2}(x_1-3)^2-\frac{1}{2}(x_2-7)^2   \right ] } +5\cdot e^{\left [ -\frac{1}{2}(x_1-7)^2-\frac{1}{2}(x_2-7)^2   \right ] }$ as the height function of the surface and $x_4 = 3\cdot e^{\left [ -\frac{1}{2}(x_1-5)^2-\frac{1}{2}(x_2-8)^2   \right ] }+3\cdot e^{\left [ -\frac{1}{2}(x_1-5)^2-\frac{1}{2}(x_2-2)^2   \right ] }$ as the resistance function of the ground. In Figure~\ref{4d6peak}, the ground resistance information at each position on the surface is characterized by the color of that point (Figure~\ref{4d-1}). It can be seen that this four-dimensional path planning problem has four height peaks and two ground resistance peaks. The route with a length of 17.18928 retrieved by the RRT*-R algorithm (Figure~\ref{4d-2}~\ref{4d-1}) successfully avoids all these six peaks, and its length is only about 0.16 error from the shortest geodesic (Figure~\ref{4d-3}) with a length of 17.02196 obtained by 600 traversals on the projection plane $(\mathbb{R}^2,h )$. 
However, the path length retrieved by the original RRT* algorithm(Figure~\ref{4d-4}) using Euclidean distance on the surface in 4-D Euclidean space is 19.438, with the same number of sampling points and the same step size as in the RRT*-R algorithm.
It is obvious that the error between the original RRT* algorithm and RRT*-R algorithm will increase significantly when the space dimension is increased from three to four dimensions.
As the dimension of the workspace increases, the RRT* algorithm may face larger errors and worse performance, especially in terms of the smoothness and optimality of the path.
To intuitively analyze the influence of resistance function of the ground $x_4$ on path planning, we conducted a three-dimensional comparative experiment (Figure~\ref{4dc-1}~\ref{4dc-2}), which has the same surface height function $x_3$ as the above four-dimensional experiment, but without ground friction resistance information (represented by the gray surface). This time, the RRT*-R algorithm chooses a route with less length, but this route will pass through the ground resistance peak of the four-dimensional experiment as shown in Figure~\ref{4dc-3} and Figure~\ref{4dcompare}. In other words, by constructing the Riemannian metric $h$ containing the information of ground friction resistance $x_{4}$ on the projection plane, RRT*-R not only follows the path with small height fluctuations, but also avoids areas where ground resistance changes dramatically. 

\begin{figure*}
   \centering
   \subfloat[]{
    \includegraphics[width=0.4\linewidth]{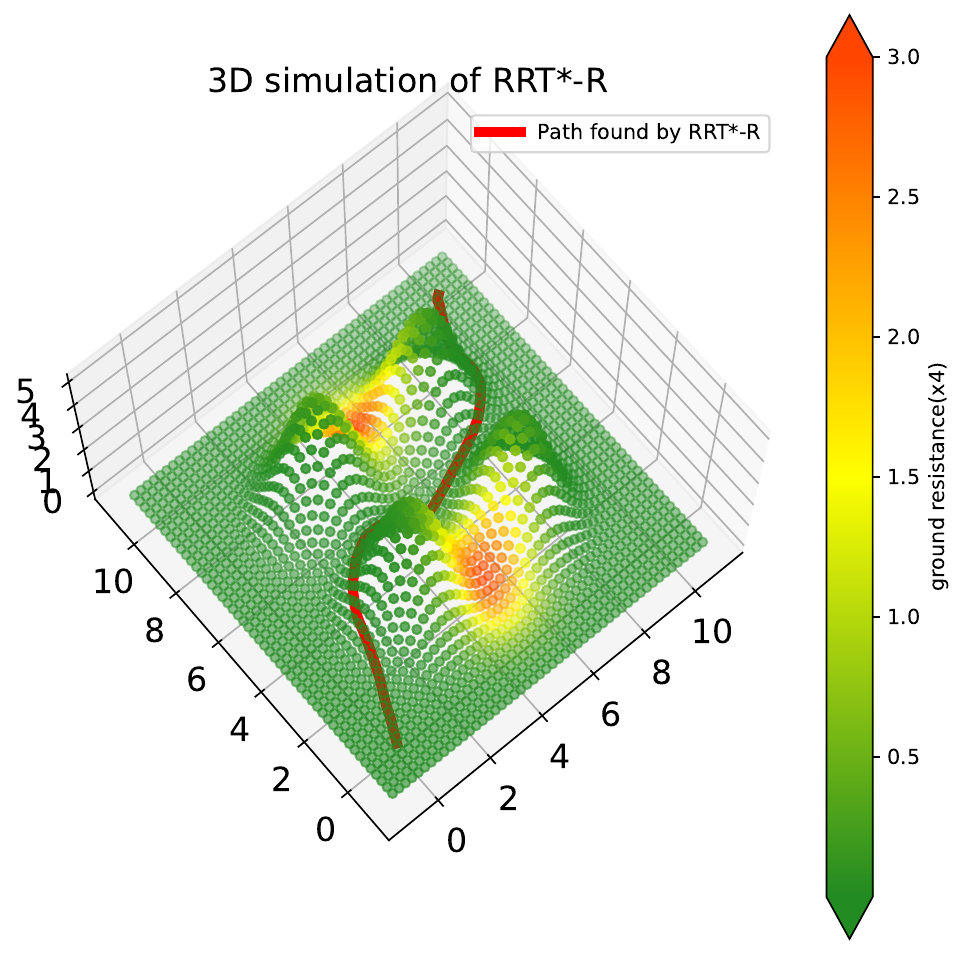}
    \label{4d-1}
    }
   \subfloat[]{
      \includegraphics[width=0.4\linewidth]{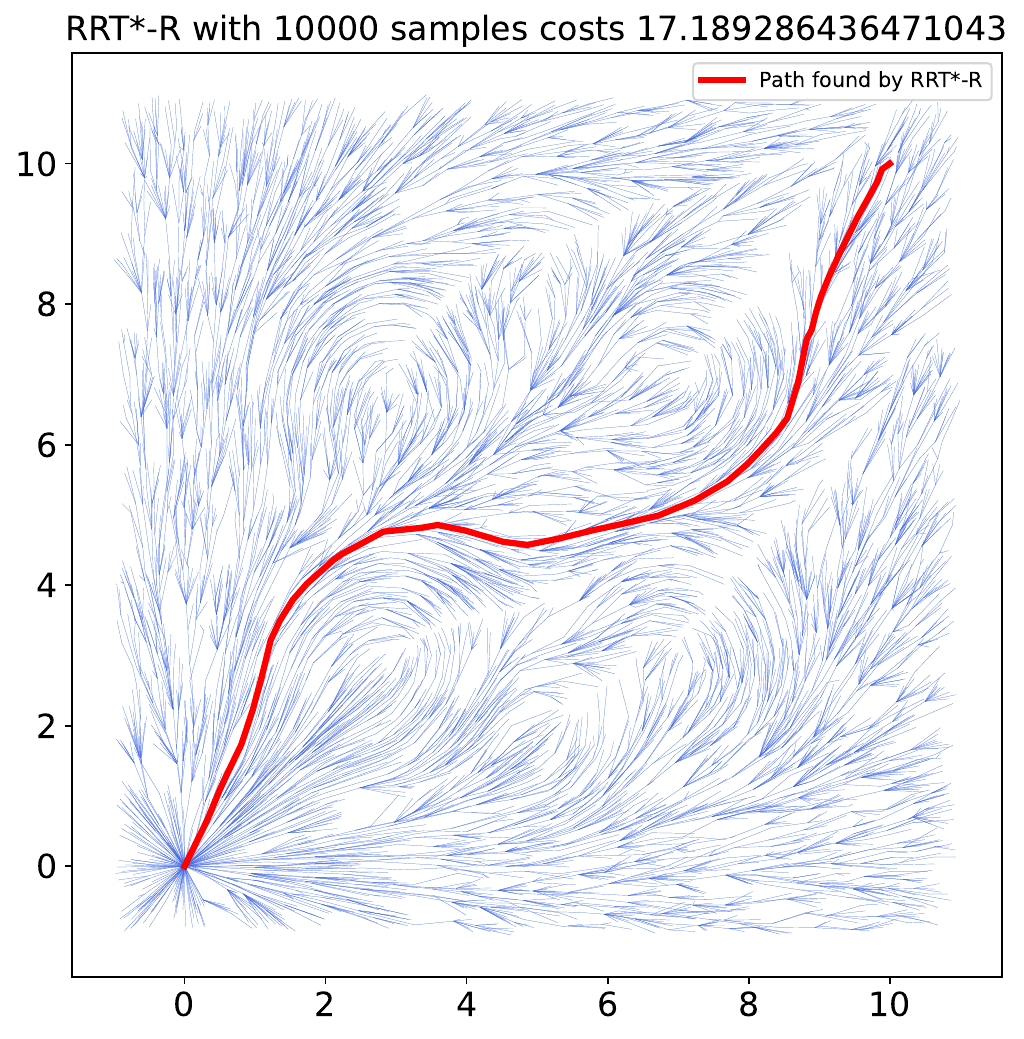}
      \label{4d-2}
    }\\
    \subfloat[]{
      \includegraphics[width=0.4\linewidth]{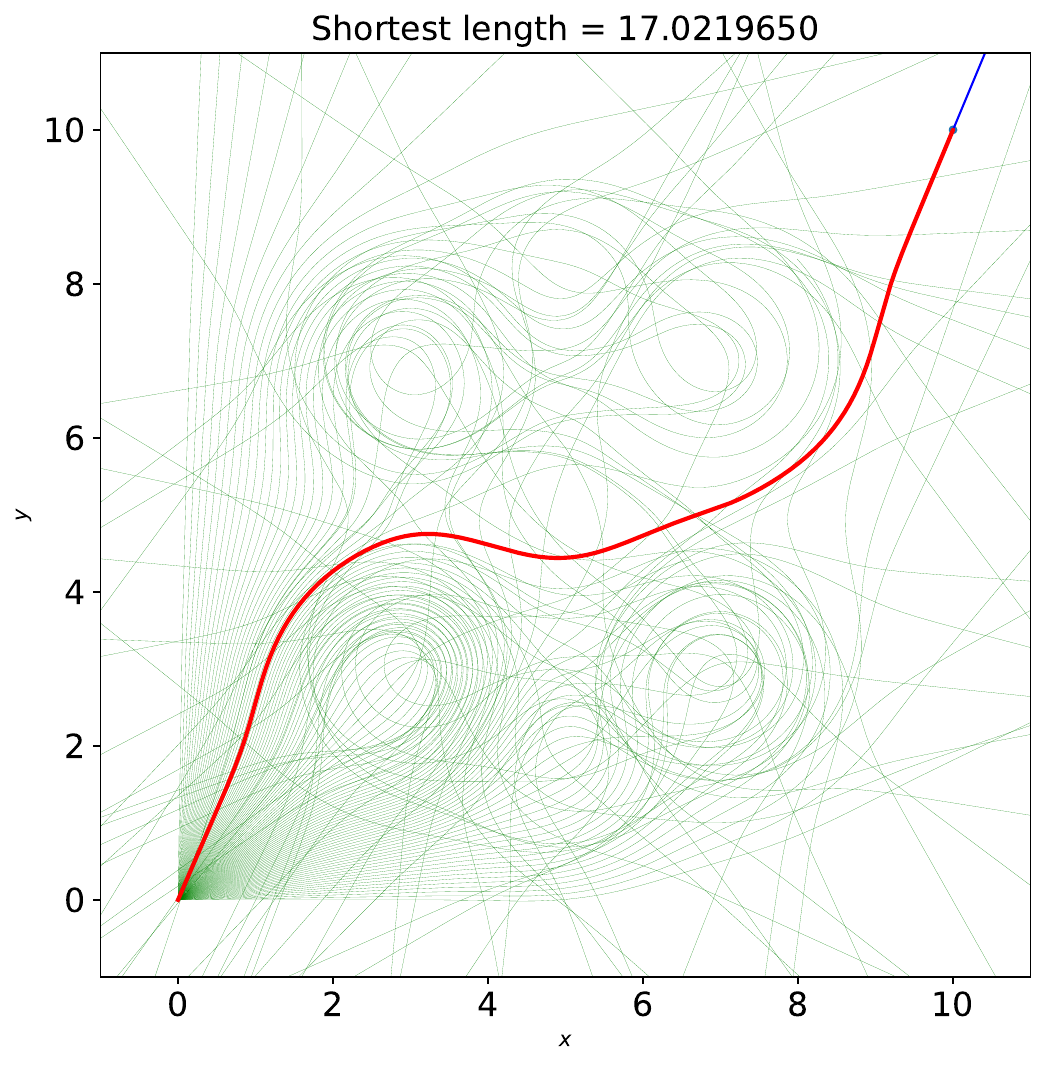}
      \label{4d-3}
    }
    \subfloat[]{
      \includegraphics[width=0.4\linewidth]{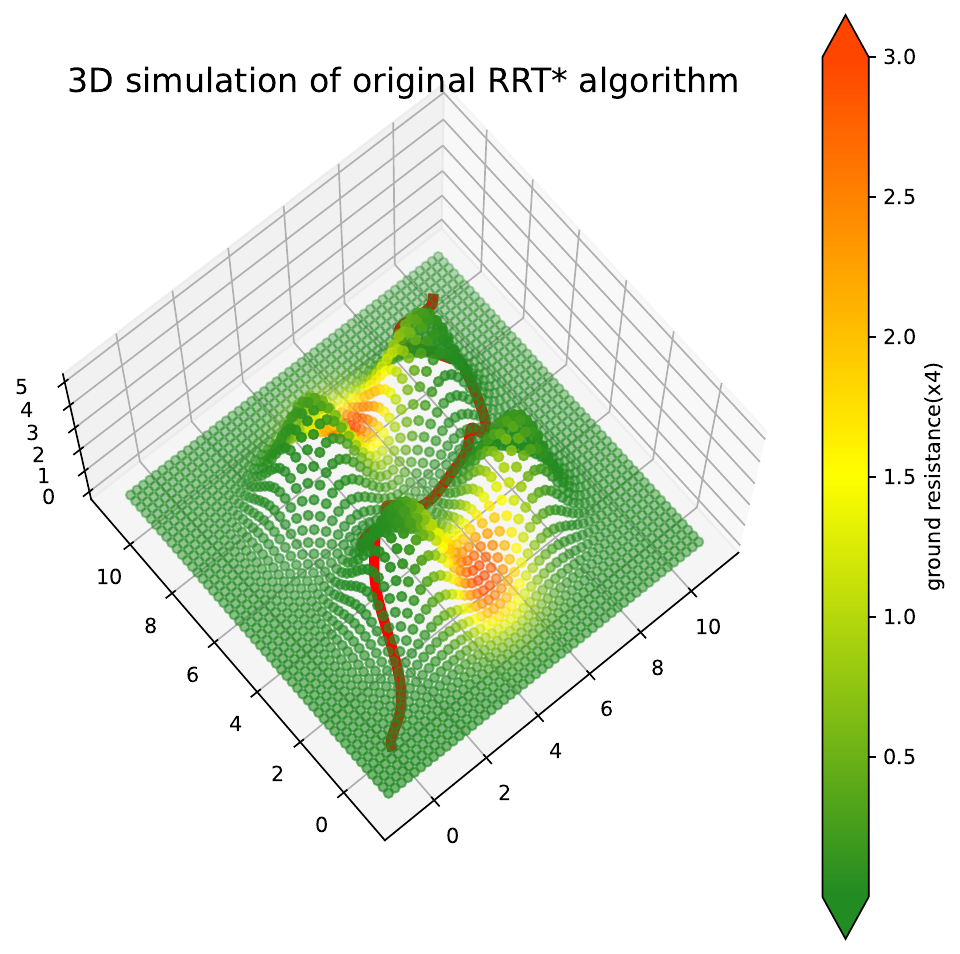}
      \label{4d-4}
    }
   \captionsetup{justification=centering} 
   \caption{Simulation with 4-D scenario (six peaks). (b) is obtained by RRT*-R taking 10000 samples. (a) is the preimage of the path in (b) under the projection map. (c) is geodesic path. (d) is obtained by original RRT* algorithm using Euclidean distance}
   \label{4d6peak}
\end{figure*}

\begin{figure*}
   \centering
   \subfloat[]{
    \includegraphics[width=0.3\linewidth]{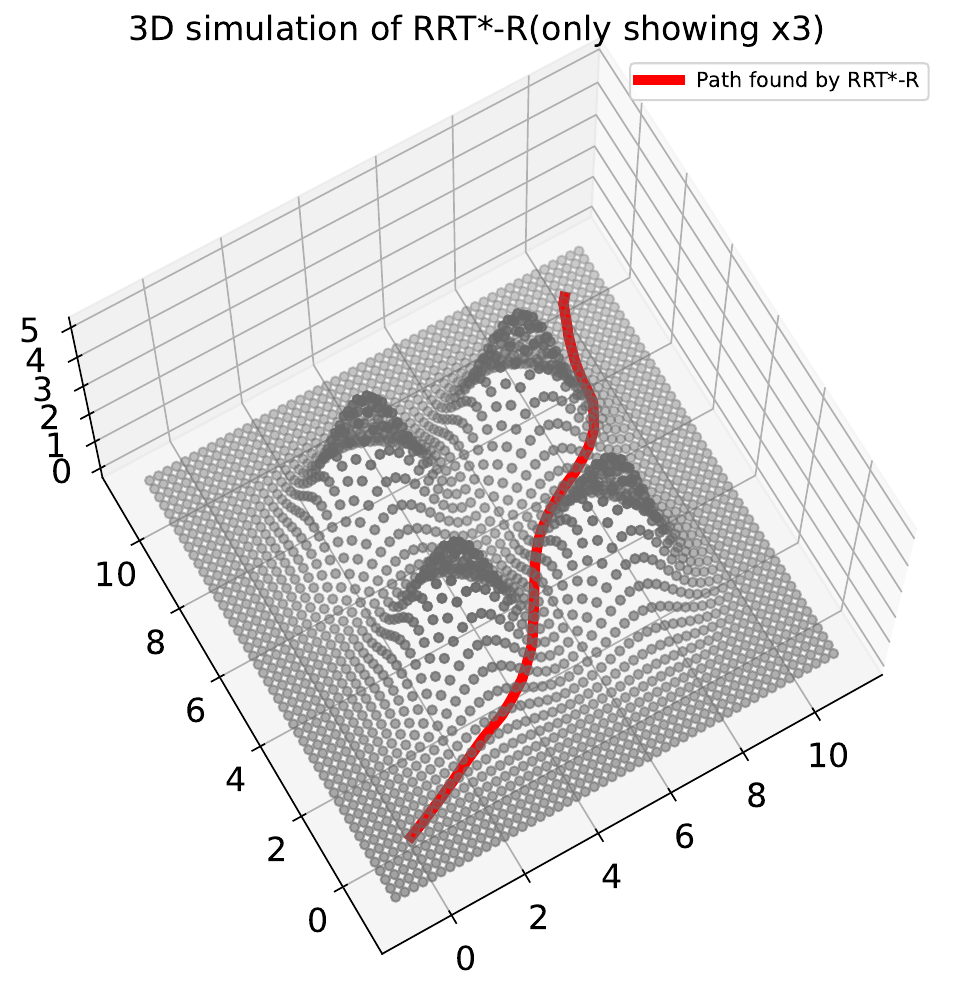}
    \label{4dc-1}
    }
   \subfloat[]{
      \includegraphics[width=0.3\linewidth]{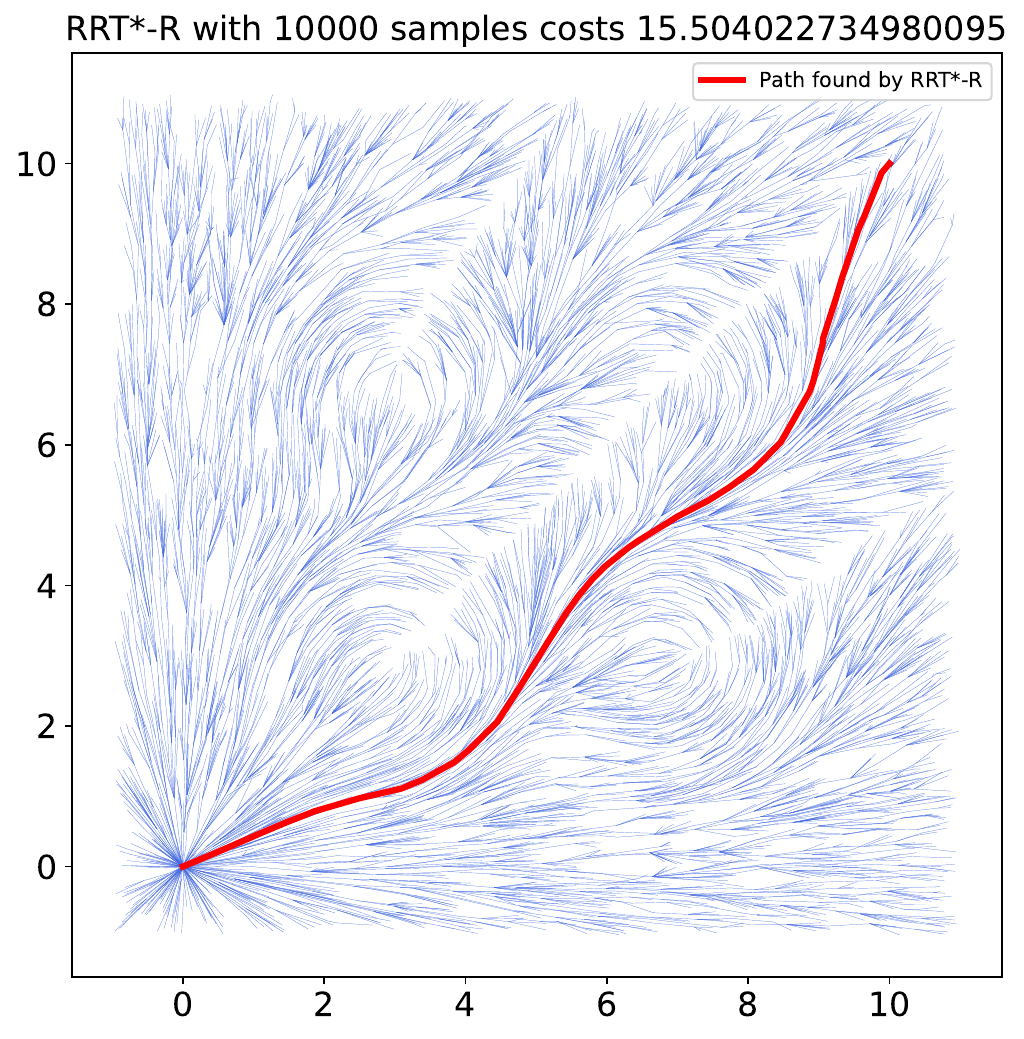}
      \label{4dc-2}
    }
    \subfloat[]{
      \includegraphics[width=0.3\linewidth]{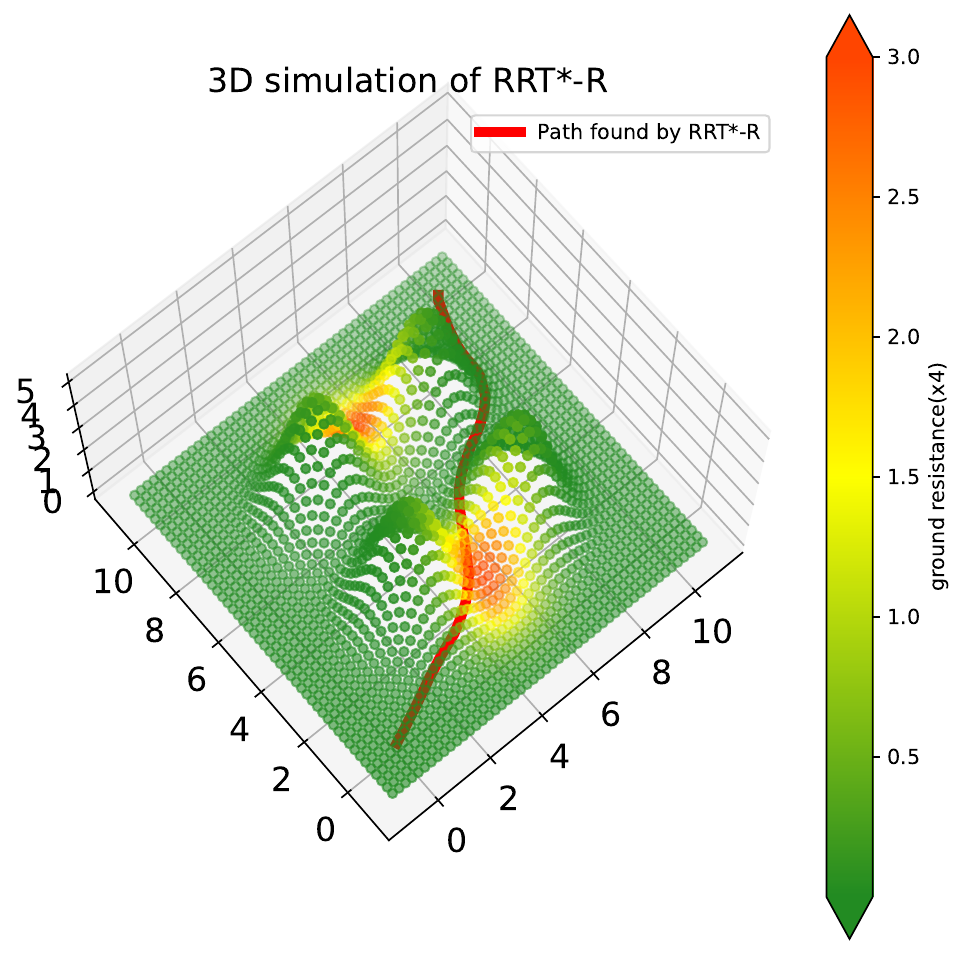}
      \label{4dc-3}
    }
   \captionsetup{justification=centering} 
   \caption{Comparative experiment of 4-D scenario. (b) is obtained by RRT*-R taking 10000 samples. (a) is the preimage of the path in (b) under the projection map. (c) is obtained by replacing the gray dot surface in (a) with a colored dot surface containing ground resistance information and keep the path in (a) unchanged. }
   \label{3d4peak}
\end{figure*}

\begin{figure}
    \centering
    \includegraphics[width=0.7\linewidth]{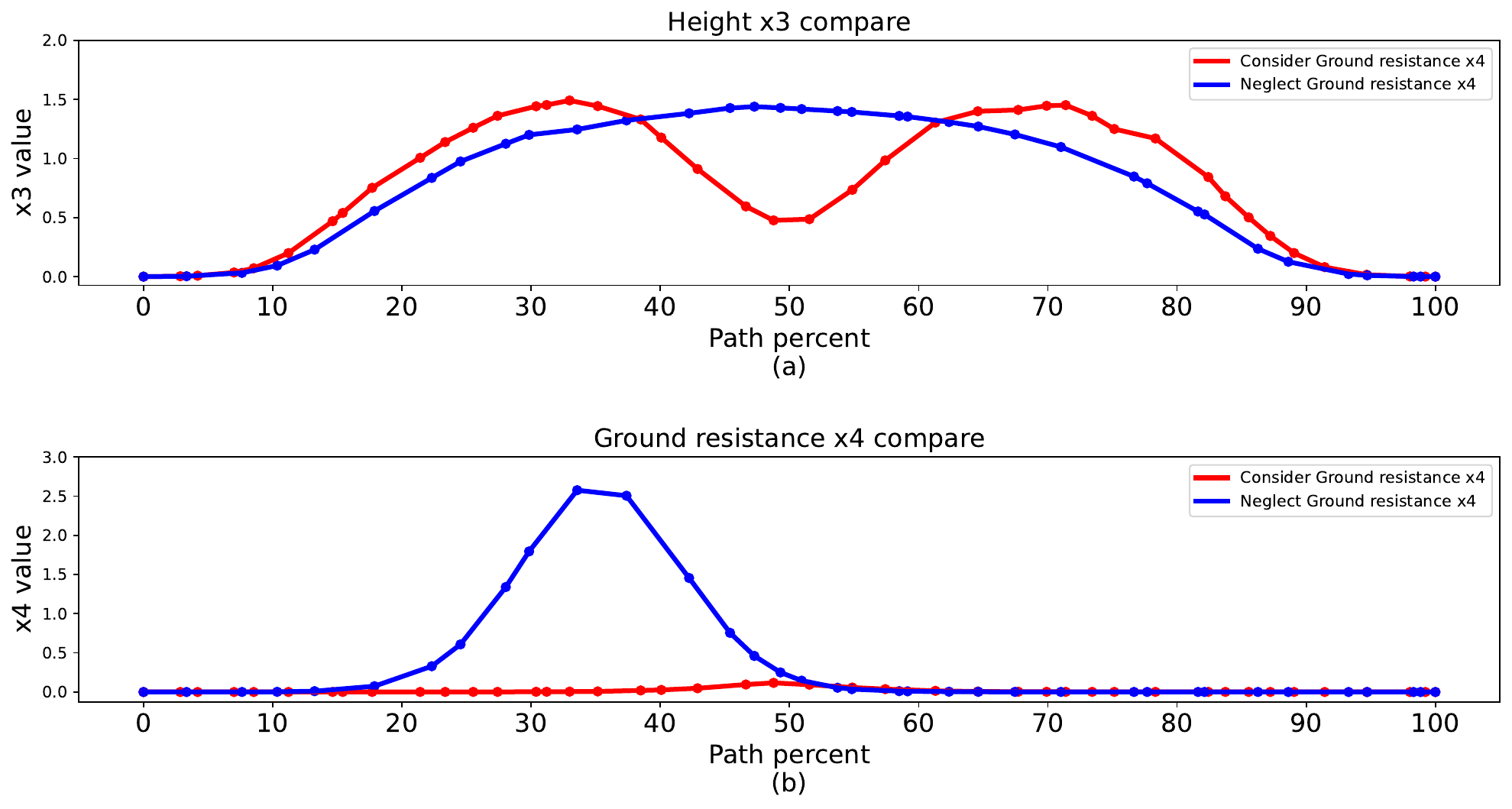}
    \caption{Variation along the path in 4-D scenario and its comparative experiment. (a) Height (b) Ground friction resistance }
    \label{4dcompare}
\end{figure}

\subsection{Analysis}
In this section, the stability and convergence properties of the RRT*-R algorithm are investigated through repeatability tests. Firstly, since the search for the optimal path of the algorithm depends on the selection of uniform sampling points, in the same experiment, we run the algorithm 150 times with the same input, compare the output length of each run, and draw a box plot to observe the distribution of the overall path length data. Figure~\ref{boxline} depicts a quantitative comparison in scenarios with different dimensions to show that the proposed method can maintain stability in different dimensions. The height function in the first 3D experiment was selected $x_3 = 5\cdot e^{\left [ -\frac{1}{2}(x_1-5)^2-\frac{1}{2}(x_2-8)^2   \right ] }+5\cdot e^{\left [ -\frac{1}{2}(x_1-8)^2-\frac{1}{2}(x_2-5)^2   \right ] }+5\cdot e^{\left [ -\frac{1}{2}(x_1-2)^2-\frac{1}{2}(x_2-2)^2   \right ] } +5\cdot e^{\left [ -\frac{1}{2}(x_1-8)^2-\frac{1}{2}(x_2-2)^2   \right ] }$, and the second 4D experiment selected the same height function with an additional resistance function of the ground $x_4 = 3\cdot e^{\left [ -\frac{1}{2}(x_1-2)^2-\frac{1}{2}(x_2-8)^2   \right ] }+3\cdot e^{\left [ -\frac{1}{2}(x_1-8)^2-\frac{1}{2}(x_2-8)^2   \right ] }+7\cdot e^{\left [ -\frac{1}{2}(x_1-5)^2-\frac{1}{2}(x_2-5)^2   \right ] }$. Figure~\ref{boxline} shows that the proposed algorithm has a small variance and dispersion in two dimensional scenarios. 

The convergence property of the RRT*-R algorithm is shown in Figure~\ref{converge}. Except for the number of sampling points, we keep other inputs unchanged, start the experiment with 2000 sampling points, and gradually increase the number of sampling points, such as 4000, 6000, until 20000. Repeat the experiment 10 times for each number of sampling points and average the output 10 path lengths. The figure shows that as the number of sampling points increases, the data output by the RRT*-R algorithm gradually converge to the path length of 15.12566 (black dotted line) obtained by traversing the geodesic. Therefore, almost all RRT*-R runs converge to the optimal solution as expected. 

\begin{figure}
    \centering
    \subfloat[]{
       \label{boxline}
       \includegraphics[width=0.4\linewidth]{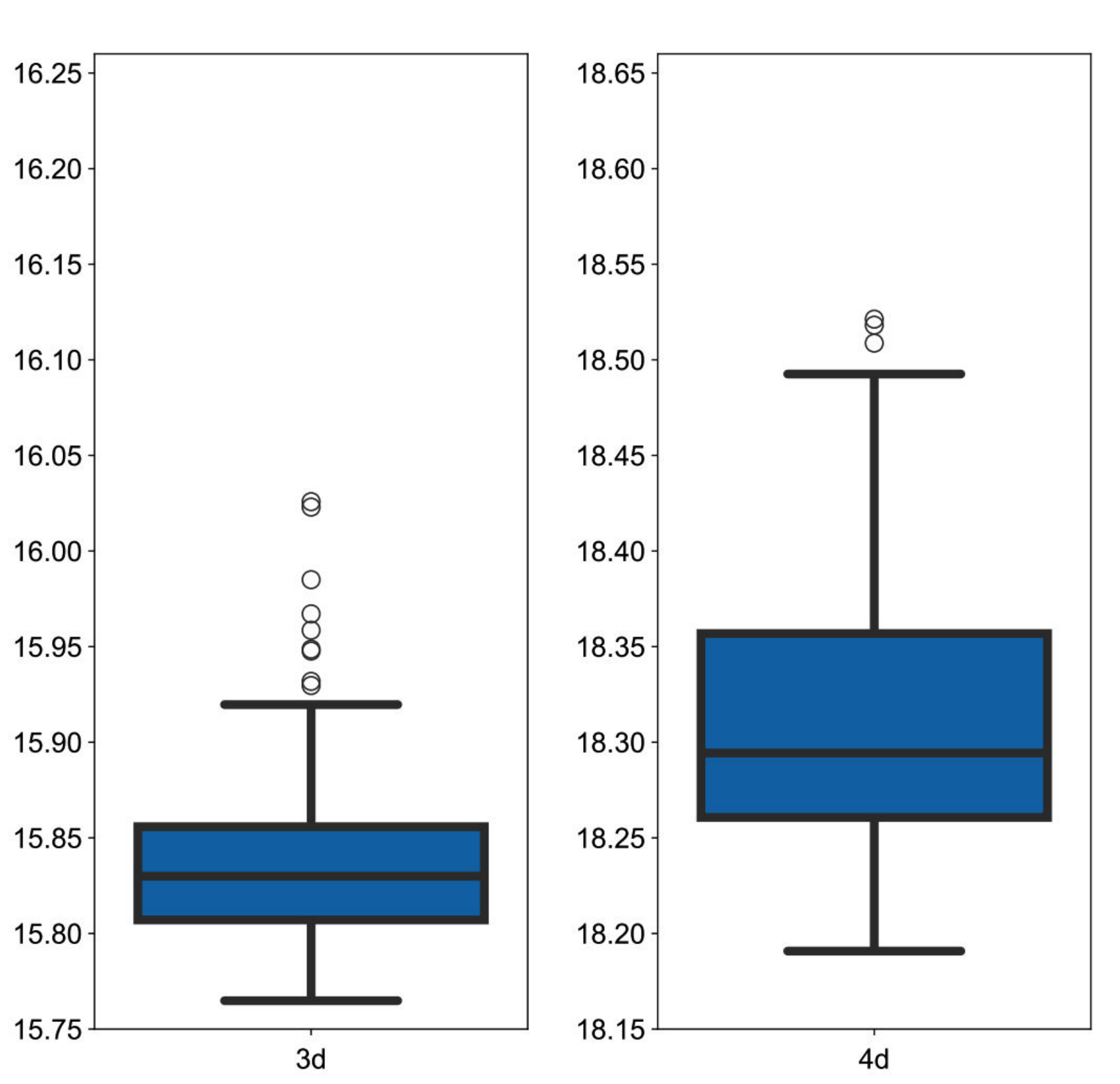}
       }
     \subfloat[]{
       \label{converge}
       \includegraphics[width=0.4\linewidth]{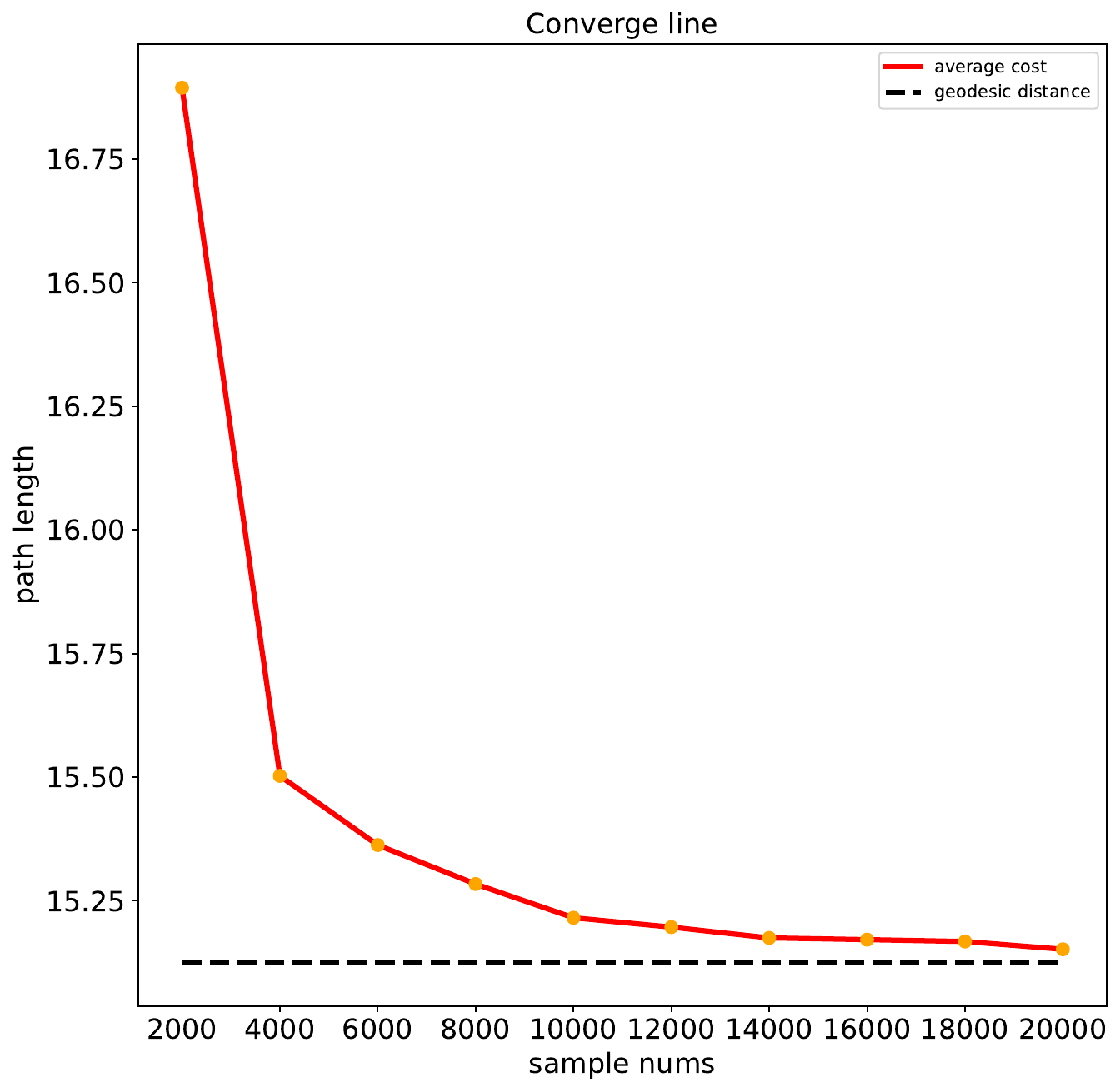}
       }
    \caption{(a) Boxplot of 150 times simulation. (b) Line plot of the path length varying with the number of sampling points.}
    \label{Analysis}
\end{figure}

It is worth mentioning that we discovered a wonderful phenomenon in the geodesic traversal experiment: when a very small perturbation occurs in the initial velocity direction of the geodesic equation, the solution of the geodesic equation, that is, shape of the geodesic,changes particularly dramatically. In other words, we find that the geometry of the geodesic is highly sensitive and dependent on the initial velocity direction. This exciting discovery will serve as a research direction for our future work.

The topic of solving the path planning problem by constructing a new Riemannian metric model is very difficult, and its related research is relatively few.
Therefore, we choose the smooth normal distribution function to simulate the raised peaks in the present paper.
For the case of non-smooth functions, we hope to explore in the subsequent research, such as using smoothing tools and other methods.
\section{Conclusion}
\label{sec:Conclusion}
In this paper, we propose a method based on Riemannian metric to solve the optimal path planning problem on the two-dimensional smooth manifold in high dimensional space.
By projecting the two-dimensional smooth manifold in high-dimensional space onto the plane $\mathbb{R}^2$, and constructing a new Riemannian metric reflecting the information of high-dimensional environment on the two-dimensional projection plane, we transform the general optimal path planning problem in the high-dimensional space into a geometric problem in a two-dimensional plane, and realize the dimension reduction of the high-dimensional problem.
And we strictly prove that the new Riemannian metric is isometric to the induced metric of submanifold by Euclidean metric on $\mathbb{R}^n$.

In addition, we propose an RRT*-R algorithm based on Riemannian metric, and carry out several simulation experiments in scenarios with different surface curvature and dimensions. The experimental results are consistent with the theory.
In order to verify the correctness and running performance of the RRT*-R algorithm, we conduct comparative experiments with the original RRT* algorithm using Euclidean distance.Under the condition that the number of sampling points and step size are the same, we find that the path retrieved by RRT*-R algorithm has better smoothness and optimization properties as the dimension of the workspace increases.
For the obstacle-free scenario satisfying geodesic completeness, we compare the experimental results of the proposed RRT*-R algorithm with the theoretical optimal path length, namely geodesic length.
We find that the error of the algorithm is very small, and lead the robot to effectively avoid the peak area where the environmental factors such as height and ground resistance changed dramatically, which verified the correctness of the algorithm.
In order to check the stability of the algorithm, we also conducted a large number of repeatability tests as well as convergence experiments with respect to the number of sampling points.



\section{Appendix}
\subsection{The case of two-dimensional smooth manifolds in $\mathbb{R}^n$}

The Christoffel symbols are defined as:
$$
\small\begin{aligned}
  \Gamma_{ij}^k =\frac{1}{2} h^{kl}\left(\frac{\partial h_{il} }{\partial x^{j} }+\frac{\partial h_{lj} }{\partial x^i}-\frac{\partial h_{ij} }{\partial x^{l} } \right) 
\end{aligned}
$$

According to Einstein summation convention, they can be calculated in detail as below.

$$\small
\begin{aligned}
\Gamma_{11}^{1}&=\frac{1}{d}\left[1+\sum\limits_{k=3}^n\left(\frac{\partial x_k}{\partial x_2}\right)^2\right]  
\cdot
\left[\sum\limits_{k=3}^n \frac{\partial x_k}{\partial x_1} \cdot\frac{\partial^2 x_k}{\partial x_1 \partial x_1}\right]
-\frac{1}{d} \left[\sum\limits_{k=3}^n\frac{\partial x_k}{\partial x_1}\cdot \frac{\partial x_k}{\partial x_2}\right]
\cdot 
\left[\sum\limits_{k=3}^n\frac{\partial^2 x_k}{\partial x_1\partial x_1}\cdot\frac{\partial x_k}{\partial x_2}\right]\\
\Gamma_{11}^{2}&=-\frac{1}{d}\left[ \sum\limits_{k=3}^n \frac{\partial x_k}{\partial x_1}\cdot\frac{\partial x_k}{\partial x_2}\right]
\cdot
\left[\sum\limits_{k=3}^n \frac{\partial x_k}{\partial x_1}\cdot \frac{\partial^2 x_k}{\partial x_1 \partial x_1}\right]
+\frac{1}{d}\left[1+\sum\limits_{k=3}^n\left(\frac{\partial x_k}{\partial x_1}\right)^2\right]
\cdot
\left[\sum\limits_{k=3}^n\frac{\partial^2 x_k}{\partial x_1 \partial x_1}\cdot\frac{\partial x_k}{\partial x_2}\right]\\
\Gamma_{12}^{1}&=\Gamma_{21}^{1} = \frac{1}{d} \left[ 1+\sum\limits_{k=3}^n \left(\frac{\partial x_k}{\partial x_2}\right)^2\right]
\cdot
\left[\sum\limits_{k=3}^n \frac{\partial x_k}{\partial x_1}\cdot \frac{\partial^2 x_k}{\partial x_1 \partial x_2}\right]
-\frac{1}{d}\left[\sum\limits_{k=3}^n\frac{\partial x_k}{\partial x_1}\cdot \frac{\partial x_k}{\partial x_2}\right]\cdot
\left[\sum\limits_{k=3}^n \frac{\partial x_k}{\partial x_2}\cdot\frac{\partial^2 x_k}{\partial x_1 \partial x_2}\right]\\
\Gamma_{12}^{2}&=\Gamma_{21}^{2} = -\frac{1}{d} \left[\sum\limits_{k=3}^n \frac{\partial x_k}{\partial x_1}\cdot\frac{\partial x_k}{\partial x_2}\right]\cdot
\left[\sum\limits_{k=3}^n \frac{\partial x_k}{\partial x_1}\cdot
\frac{\partial^2 x_k}{\partial x_1 \partial x_2}\right]
+\frac{1}{d} \left[1+\sum\limits_{k=3}^n \left(\frac{\partial x_k}{\partial x_1}\right)^2\right]\cdot
\left[\sum\limits_{k=3}^n \frac{\partial x_k}{\partial x_2}\cdot
\frac{\partial^2 x_k}{\partial x_1 \partial x_2}\right]
\end{aligned}$$
$$
\begin{aligned}
\Gamma_{22}^{1}&= \frac{1}{d} \left[1+\sum\limits_{k=3}^n \left(\frac{\partial x_k}{\partial x_2}\right)^2\right]\cdot
\left[\sum\limits_{k=3}^n \frac{\partial x_k}{\partial x_1}\cdot
\frac{\partial^2 x_k}{\partial x_2 \partial x_2}\right]
-\frac{1}{d} \left[\sum\limits_{k=3}^n \frac{\partial x_k}{\partial x_1}\cdot\frac{\partial x_k}{\partial x_2}\right]\cdot
\left[\sum\limits_{k=3}^n \frac{\partial x_k}{\partial x_2}\cdot
\frac{\partial^2 x_k}{\partial x_2 \partial x_2}\right]\\
\Gamma_{22}^{2}&= -\frac{1}{d} \left[\sum\limits_{k=3}^n \frac{\partial x_k}{\partial x_1}\cdot \frac{\partial x_k}{\partial x_2}\right]\cdot
\left[\sum\limits_{k=3}^n \frac{\partial x_k}{\partial x_1}\cdot \frac{\partial^2 x_k}{\partial x_2 \partial x_2}\right]
+\frac{1}{d}\left[1+\sum\limits_{k=3}^n \left(\frac{\partial x_k}{\partial x_1}\right)^2\right]\cdot
\left[\sum\limits_{k=3}^n \frac{\partial x_k}{\partial x_2}\cdot \frac{\partial^2  x_k}{\partial x_2 \partial x_2}\right]
\end{aligned}$$

Substitute the above eight symbols into the geodesic equation:
$${\small\left\{\begin{aligned}
\frac{d x_1}{d t}=&y_1 \\
\frac{d x_2}{d t}=&y_2 \\
\frac{d y_1}{d t}=&-\left(y_1\right)^2 \Gamma_{11}^{1}-2 y_1 y_2 \Gamma_{12}^{1}-\left(y_2\right)^2 \Gamma_{22}^{1} \\
\frac{d y_2}{d t}=&-\left(y_1\right)^2 \Gamma_{11}^{2}-2 y_1 y_2 \Gamma_{12}^{2}-\left(y_2\right)^2 \Gamma_{22}^{2}
\end{aligned}\right.}$$

\subsection{The case of two-dimensional smooth manifolds in $\mathbb{R}^3$}

Denote the normal distribution surface as $x_3=e^{-\left(x_1^2+x_2^2\right)}$. Calculate the Christoffel symbol as follows.
$$
\small\begin{aligned}
  \Gamma_{11}^1 &=\frac{-8 x_1^3+4 x_1}{e^{2\left(x_1^2+x_2^2\right)}+4 x_1^2+4 x_2^2} ,
  \Gamma_{11}^2 =\frac{-8 x_1^2 x_2+4 x_2}{e^{2\left(x_1^2+x_2^2\right)}+4 x_1^2+4 x_2^2}\\
  \Gamma_{12}^1 &=\Gamma _{21}^{1}=\frac{-8 x_1^2 x_2}{e^{2\left(x_1^2+x_2^2\right)}+4 x_1^2+4 x_2^2},
  \Gamma_{12}^2 =\Gamma_{21}^2=\frac{-8 x_1 x_2^2}{e^{2\left(x_1^2+x_2^2\right)}+4 x_1^2+4 x_2^2}\\
  \Gamma_{22}^1 &=\frac{-8 x_1 x_2^2+4 x_1}{e^{2\left(x_1^2+x_2^2\right)}+4 x_1^2+4 x_2^2},
  \Gamma_{22}^2 =\frac{-8 x_2^3+4 x_2}{e^{2\left(x_1^2+x_2^2\right)}+4 x_1^2+4 x_2^2}
\end{aligned}
$$

A smooth curve $\gamma: I \rightarrow \mathbb{R}^2$ on this surface is geodesic if and only if the curve satisfies geodesic equation below
$$\small{\left\{\begin{aligned}
\frac{d x_1}{d t}=&y_1 \\
\frac{d x_2}{d t}=&y_2 \\
\frac{d y_1}{d t}=&-\left(y_{1}\right)^2 \cdot \frac{-8 x_1^3+4 x_1}{e^{2\left(x_1^2 + x_2^2\right)}+4 x_1^2+4 x_2^2}
-2 y_1 y_2 \cdot \frac{-8 x_1^2 x_2}{e^{2\left(x_1^2+x_2^2\right)}+4 x_1^2+4 x_2^2}
-\left(y_2\right)^2 \cdot \frac{-8 x_1 x_2^2+4 x_1}{e^{2\left(x_1^2+x_2^2\right)}+4 x_1^2+4 x_2^2} \\
\frac{d y_2}{d t}=&-\left(y_1\right)^2 \cdot \frac{-8 x_1^2 x_2+4 x_2}{e^{2\left(x_1^2+x_2^2\right)}+4 x_1^2+4 x_2^2}
-2 y_1 y_2 \cdot \frac{-8 x_1 x_2^2}{e^{2\left(x_1^2+x_2^2\right)}+4 x_1^2+4 x_2^2}
-\left(y_2\right)^2 \cdot \frac{-8 x_2^3+4 x_2}{e^{2\left(x_1^2+x_2^2\right)}+4 x_1^2+4 x_2^2}
\end{aligned}\right.} $$

\bibliographystyle{elsarticle-num} 
\bibliography{ref}





\end{document}